\documentclass{article} 
\usepackage{nips15submit_e,times}
\usepackage{hyperref}
\usepackage{url}
\usepackage{amsmath}
\usepackage{amsfonts}

\usepackage[numbers,sort&compress]{natbib}

\usepackage{amssymb}
\usepackage{amsthm}

\usepackage{tikz}

\usepackage{wrapfig}
\usepackage{mathtools}

\usepackage{graphicx} 
\usepackage{subfig}
\graphicspath{{figures/}}

\usepackage[lined,boxed,commentsnumbered,ruled]{algorithm2e}

\newcommand{\RN}[1]{%
  \textup{\lowercase\expandafter{\it \romannumeral#1}}%
}

\input{Definitions}

\ifx\proof\undefined
\newenvironment{proof}{\par\noindent{\bf Proof\ }}{\hfill\BlackBox\\[2mm]}
\fi

\ifx\theorem\undefined
\newtheorem{theorem}{Theorem}
\fi
\ifx\example\undefined
\newtheorem{example}{Example}
\fi
\ifx\lemma\undefined
\newtheorem{lemma}[theorem]{Lemma}
\fi

\ifx\corollary\undefined
\newtheorem{corollary}[theorem]{Corollary}
\fi

\ifx\assumption\undefined
\newtheorem{assumption}{Assumption}
\fi

\ifx\definition\undefined
\newtheorem{definition}{Definition}
\fi

\ifx\proposition\undefined
\newtheorem{proposition}[theorem]{Proposition}
\fi

\ifx\remark\undefined
\newtheorem{remark}[theorem]{Remark}
\fi

\ifx\conjecture\undefined
\newtheorem{conjecture}[theorem]{Conjecture}
\fi

\ifx\factoid\undefined
\newtheorem{factoid}[theorem]{Fact}
\fi

\ifx\axiom\undefined
\newtheorem{axiom}[theorem]{Axiom}
\fi

\title{On the Convergence of Stochastic Gradient MCMC Algorithms with High-Order Integrators}

\author{
Changyou Chen$^\dag$~~~~~~~~~~~~ Nan Ding$^\ddag$ ~~~~~~~~~~~~ Lawrence Carin$^\dag$ \\
$^\dag$Dept. of Electrical and Computer Engineering, Duke University, Durham, NC, USA \\
$^\ddag$Google Inc., Venice, CA, USA \\
\texttt{cchangyou@gmail.com; dingnan@google.com; lcarin@duke.edu}
}

\nipsfinalcopy 

\begin{document}

\maketitle

\begin{abstract}
Recent advances in Bayesian learning with large-scale data have witnessed emergence 
of stochastic gradient MCMC algorithms (SG-MCMC), such as stochastic gradient 
Langevin dynamics (SGLD), stochastic gradient Hamiltonian MCMC (SGHMC), and 
the stochastic gradient thermostat. While finite-time convergence properties of 
the SGLD with a 1st-order Euler integrator have recently been studied, corresponding theory for 
general SG-MCMCs has not been explored. In this paper we consider general SG-MCMCs 
with high-order integrators, and develop theory to analyze finite-time convergence properties 
and their asymptotic invariant measures. Our theoretical results show faster convergence 
rates and more accurate invariant measures for SG-MCMCs with higher-order integrators. 
For example, with the proposed efficient 2nd-order symmetric splitting integrator, the 
{\em mean square error} (MSE) of the posterior average for the SGHMC achieves an optimal 
convergence rate of $L^{-4/5}$ at $L$ iterations, compared to $L^{-2/3}$ for the SGHMC and 
SGLD with 1st-order Euler integrators. Furthermore, convergence results of decreasing-step-size 
SG-MCMCs are also developed, with the same convergence rates as their fixed-step-size 
counterparts for a specific decreasing sequence. Experiments on both synthetic and real 
datasets verify our theory, and show advantages of the proposed method in two large-scale 
real applications.
\end{abstract}

\section{Introduction}\vspace{-0.3cm}
In large-scale Bayesian learning, diffusion based sampling methods have become increasingly 
popular. Most of these methods are based on It\^{o} diffusions, defined as:
\begin{align}
\mathrm{d}\Xb_t &= F(\Xb_t)\mathrm{d}t + \sigma(\Xb_t)\mathrm{d}W_t~. \label{eq:itodiffusion}
\end{align}
Here $\Xb_t \in \RR^n$ represents model states, $t$ the time index, $W_t$ is Brownian motion, 
functions $F: \RR^n \to \RR^n$ and $\sigma: \RR^n \rightarrow \RR^{n\times m}$ ($m$ not necessarily
equal to $n$) are assumed 
to satisfy the usual Lipschitz continuity condition. 

In a Bayesian setting, the goal is to design appropriate functions $F$ and $\sigma$, so that the 
stationary distribution, $\rho(\Xb)$, of the It\^{o} diffusion has a marginal distribution that is equal 
to the posterior distribution of interest. For example, 1st-order Langevin dynamics (LD) 
correspond to $\Xb = \thetab$,  $F = -\nabla_{\thetab} U$ and $\sigma = \sqrt{2}\Ib_n$, with $\Ib_n$ 
being the $n \times n$ identity matrix; 2nd-order Langevin dynamics correspond to 
$\Xb = (\thetab, \pb)$, $F= \Big( \begin{array}{c}
		\pb \\
		-D \pb-\nabla_\thetab U \end{array} \Big)$,
and $\sigma = \sqrt{2D}\Big( \begin{array}{cc}
		{\bf 0} & {\bf 0} \\
		{\bf 0} & \Ib_n \end{array} \Big)$ for some $D > 0$. Here $U$ is the unnormalized negative
log-posterior, and $\pb$ is known as the momentum \cite{ChenFG:ICML14,DingFBCSN:NIPS14}. 
Based on the Fokker-Planck equation \cite{Risken:FPE89}, the stationary distributions 
of these dynamics exist and their marginals over $\thetab$ are equal to 
$\rho(\thetab) \propto \exp(-U(\thetab))$, the posterior distribution we are interested in.

Since It\^{o} diffusions are continuous-time Markov processes, exact sampling is in general 
infeasible. As a result, the following two approximations have been introduced in the 
machine learning literature \cite{WellingT:ICML11,ChenFG:ICML14,DingFBCSN:NIPS14}, 
to make the sampling numerically feasible and practically scalable: 1) Instead of analytically 
integrating infinitesimal increments $\mathrm{d}t$, numerical integration over small step $h$ 
is used to approximate the integration of the true dynamics. Although many numerical 
schemes have been studied in the SDE literature, in machine learning only the 1st-order 
Euler scheme is widely applied. 2) During every integration, instead of working with the 
gradient of the full negative log-posterior $U(\thetab)$, a stochastic-gradient version of it, $\tilde{U}_l(\thetab)$, 
is calculated from the $l$-th minibatch of data, important when considering problems with 
massive data. In this paper, we call algorithms based on 1) and 2) SG-MCMC algorithms. 
To be complete, some recently proposed SG-MCMC algorithms are briefly reviewed in 
Appendix~\ref{sec:sgmcmc_review}. SG-MCMC algorithms often work well in practice, 
however some theoretical concerns about the convergence properties have been 
raised \cite{TehTV:arxiv14,VollmerZT:arxiv15,Betancourt:arxiv15}.

Recently, \cite{SatoN:ICML14,TehTV:arxiv14,VollmerZT:arxiv15} showed that  the 
SGLD \cite{WellingT:ICML11} converges weakly to the true posterior. 
In \cite{Betancourt:arxiv15}, the author studied the sample-path inconsistency of 
the Hamiltonian PDE with stochastic gradients (but not the SGHMC), and pointed 
out its incompatibility with data subsampling. However, real applications only require 
convergence in the weak sense, {\it i.e.}, instead of requiring sample-wise convergence 
as in \cite{Betancourt:arxiv15}, only laws of sample paths are of concern\footnote{For 
completeness, we provide mean sample-path properties of the SGHMC (similar 
to \cite{Betancourt:arxiv15}) in Appendix~\ref{sec:meanflow}.}. Very recently, the invariance 
measure of an SG-MCMC with a specific stochastic gradient noise was studied 
in \cite{LeimkuhlerS:arxiv15}. However, the technique is not readily applicable to our general setting.

In this paper we focus on general SG-MCMCs, and study the role of their numerical 
integrators. Our main contributions include: $\RN{1}$) From a theoretical viewpoint, 
we prove weak convergence results for general SG-MCMCs, which are of practical interest. 
Specifically, for a $K$th-order numerical integrator, the bias of the expected sample average 
of an SG-MCMC at iteration $L$ is upper bounded by $L^{-K/(K+1)}$ with optimal step size 
$h \propto L^{-1/(K+1)}$, and the MSE by $L^{-2K/(2K+1)}$ with optimal $h \propto L^{-1/(2K+1)}$. 
This generalizes the results of the SGLD with an Euler integrator ($K = 1$) 
in \cite{SatoN:ICML14,TehTV:arxiv14,VollmerZT:arxiv15}, 
and is better when $K \ge 2$; $\RN{2}$) From a practical perspective, we introduce a numerically 
efficient 2nd-order integrator, based on symmetric splitting schemes \cite{LeimkuhlerS:arxiv15}. 
When applied to the SGHMC, it outperforms existing algorithms, including the SGLD and SGHMC 
with Euler integrators, considering both synthetic and large real datasets. 

\section{Preliminaries \& Two Approximation Errors in SG-MCMCs}
\label{sec:sgmcmc}

In weak convergence analysis, instead of working directly with sample-paths in \eqref{eq:itodiffusion}, 
we study how the expected value of any suitably smooth statistic of $\Xb_t$ evolves in time. 
This motivates the introduction of an (infinitesimal) {\em generator}. Formally, the {\em generator} 
$\mathcal{L}$ of the diffusion \eqref{eq:itodiffusion} is defined for any compactly supported twice 
differentiable function $f: \RR^n \rightarrow \RR$, such that,
\begin{align*}
	\mathcal{L}f(\Xb_t) &\triangleq \lim_{h \rightarrow 0^{+}} \frac{\mathbb{E}\left[f(\Xb_{t+h})\right] - f(\Xb_t)}{h} 
	= \rbr{F(\Xb_t) \cdot \nabla + \frac{1}{2}\left(\sigma(\Xb_t) \sigma(\Xb_t)^T\right)\!:\! \nabla \nabla^T} f(\Xb_t)~,
\end{align*}
where $\ab\cdot \bb \triangleq \ab^T\bb$, $\Ab\!:\!\Bb \triangleq \text{tr}(\Ab^T \Bb)$,
$h\rightarrow 0^{+}$ means $h$ approaches zero along the positive real axis.
$\mathcal{L}$ is associated with an integrated 
form via Kolmogorov's backward equation\footnote{More details of the equation are provided 
in Appendix~\ref{sec:KBE}. Specifically, under mild conditions on $F$, we can expand the 
operator $e^{h\mathcal{L}}$ up to the $m$th-order ($m \geq 1$) such that the remainder terms 
are bounded by $O(h^{m+1})$. Refer to \cite{AbdulleVZ:SIAMJNA15} for more details. We will 
assume these conditions to hold for the $F$'s in this paper. } :
$\mathbb{E}\left[f(\Xb_T^e)\right] = e^{T\mathcal{L}}f(\Xb_0)$, where $\Xb_T^e$ denotes the 
exact solution of the diffusion \eqref{eq:itodiffusion} at time $T$. The operator $e^{T\mathcal{L}}$ is called 
the Kolmogorov operator for the diffusion \eqref{eq:itodiffusion}. Since diffusion \eqref{eq:itodiffusion} 
is continuous, it is generally infeasible to solve analytically (so is $e^{T\mathcal{L}}$). In practice, 
a local numerical integrator is used for every small step $h$, with the corresponding Kolmogorov 
operator $P_h$ approximating $e^{h\mathcal{L}}$. Let $\Xb^n_{lh}$ denote the approximate 
sample path from such a {\bf n}umerical integrator; similarly, we have 
$\EE[f(\Xb^n_{lh})] = P_h f(\Xb^n_{(l-1)h})$. Let $\mathcal{A} \circ \mathcal{B}$ denote the 
{\em composition} of two operators $\mathcal{A}$ and $\mathcal{B}$, {\it i.e.}, $\mathcal{A}$ 
is evaluated on the output of $\mathcal{B}$. For time $T = Lh$, we have the following approximation
\begin{align*}
\mathbb{E}\left[f(\Xb^e_T)\right] \stackrel{\mathclap{A_1}}{=} e^{h\mathcal{L}} \circ \ldots \circ e^{h\mathcal{L}}f(\Xb_0) \stackrel{\mathclap{A_2}}{\simeq} P_h \circ \ldots \circ P_h f(\Xb_0)= \EE[f(\Xb_T^n)],
\end{align*} 
with $L$ compositions, where $A_1$ is obtained by decomposing $T\mathcal{L}$ into $L$ 
sub-operators, each for a minibatch of data, while approximation $A_2$ is manifested by 
approximating the infeasible $e^{h\mathcal{L}}$ with $P_h$ from a feasible integrator, 
{\it e.g.}, the symmetric splitting integrator proposed later, such that $\mathbb{E}\left[f(\Xb^n_T)\right]$ 
is close to the exact expectation $\mathbb{E}\left[f(\Xb^e_T)\right]$. The latter is the first 
approximation error introduced in SG-MCMCs. Formally, to characterize the degree of 
approximation accuracies for different numerical methods, we use the following definition.
\begin{definition}\label{def:k-order-integrator}
An integrator is said to be a $K$th-order local integrator if for any smooth and bounded function $f$,
the corresponding Kolmogorov operator $P_{h}$ satisfies the following relation:
\begin{align}
	P_h f(\xb) = e^{h\mathcal{L}} f(\xb) + O(h^{K+1})~.
\end{align}
\end{definition}

The second approximation error is manifested when  handling large data. Specifically, the 
SGLD and SGHMC use stochastic gradients in the 1st and 2nd-order LDs, respectively, 
by replacing in $F$ and $\Lcal$ the full negative log-posterior $U$ with a scaled log-posterior, 
$\tilde{U}_l$, from the $l$-th minibatch. We denote the corresponding generators with stochastic 
gradients as $\tilde{\Lcal}_l$, {\it e.g.}, the generator in the $l$-th minibatch for the SGHMC 
becomes $\tilde{\Lcal}_l = \Lcal+\Delta V_l$,  where 
$\Delta V_l = (\nabla_{\theta}\tilde{U}_l - \nabla_{\theta} U) \cdot \nabla_{p}$. As a result, 
in SG-MCMC algorithms, we use the noisy operator $\tilde{P}_h^l$ to approximate 
$e^{h\tilde{\Lcal}_l}$ such that $\EE[f(\Xb^{n,s}_{lh})] = \tilde{P}_h^l f(\Xb_{(l-1)h})$, 
where $\Xb^{n,s}_{lh}$ denotes the {\bf n}umerical sample-path with {\bf s}tochastic gradient 
noise, {\it i.e.},
\begin{align}\label{eq:sgmcmc_integrator}
\mathbb{E}\left[f(\Xb_T^e)\right] 
 \stackrel{\mathclap{B_1}}{\simeq} e^{h\tilde{\Lcal}_L} \circ \ldots \circ e^{h\tilde{\Lcal}_1}f(\Xb_0) \stackrel{\mathclap{B_2}}{\simeq} \tilde{P}^L_h \circ \ldots \circ \tilde{P}^1_h f(\Xb_0) = \mathbb{E}[f(\Xb_T^{n,s})].
\end{align}
Approximations $B_1$ and $B_2$ in \eqref{eq:sgmcmc_integrator} are from the \emph{stochastic gradient} 
and {\emph{numerical integrator} approximations, respectively. Similarly, we say $\tilde{P}^l_h$ 
corresponds to a $K$th-order local integrator of $\tilde{\Lcal}_l$ if 
$\tilde{P}^l_h f(\xb) = e^{h\tilde{\Lcal}_l} f(\xb) + O(h^{K+1})$. In the following sections, we focus 
on SG-MCMCs which use numerical integrators with stochastic gradients, and for the first time analyze 
how the two introduced errors affect their convergence behaviors. For notational simplicity, we 
henceforth use $\Xb_t$ to represent the approximate sample-path $\Xb_t^{n,s}$.

\section{Convergence Analysis}
\label{sec:convergence}

This section develops theory to analyze finite-time convergence properties of general SG-MCMCs 
with both fixed and decreasing step sizes, as well as their asymptotic invariant measures.

\subsection{Finite-time error analysis}\label{sec:finitetimeerr}

Given an ergodic\footnote{See \cite{Hasminskii:book12,VollmerZT:arxiv15} for conditions to ensure
\eqref{eq:itodiffusion} is ergodic.} It\^{o} diffusion \eqref{eq:itodiffusion} with an invariant measure 
$\rho(\xb)$, the posterior average is defined as:
$\bar{\phi} \triangleq \int_{\mathcal{X}} \phi(\xb) \rho(\xb) \mathrm{d}\xb$ 
for some test function $\phi(\xb)$ of interest. For a given numerical method with generated samples 
$(\Xb_{lh})_{l=1}^L$, we use the {\em sample average} $\hat{\phi}$ defined as 
$\hat{\phi} = \frac{1}{L} \sum_{l = 1}^L \phi(\Xb_{lh})$
to approximate $\bar{\phi}$. In the analysis,
we define a functional $\psi$ that solves the following \emph{Poisson Equation}:
\begin{align}\label{eq:PoissonEq1}
	\mathcal{L} \psi(\Xb_{lh}) =  \phi(\Xb_{lh}) - \bar{\phi}, \;\; \text{or equivalently,} \;\;\; \frac{1}{L}\sum_{l=1}^L \mathcal{L} \psi(\Xb_{lh}) = \hat{\phi} - \bar{\phi}.
\end{align}
The solution functional $\psi(\Xb_{lh})$ characterizes the difference between $\phi(\Xb_{lh})$ 
and the posterior average $\bar{\phi}$ for every $\Xb_{lh}$, thus would typically possess a unique 
solution, which is at least as smooth as $\phi$ under the elliptic or hypoelliptic settings \cite{MattinglyST:JNA10}. 
In the unbounded domain of $\Xb_{lh} \in \RR^n$, to make the presentation simple, we 
follow \cite{VollmerZT:arxiv15} and make certain assumptions on the solution functional, $\psi$, of 
the Poisson equation \eqref{eq:PoissonEq1}, which are used in the detailed proofs. Extensive 
empirical results have indicated the assumptions to hold in many real applications, though extra 
work is needed for theoretical verifications for different models, which is beyond the scope of this paper.

\begin{assumption}\label{ass:assumption1}
$\psi$ and its up to 3rd-order derivatives, $\mathcal{D}^k \psi$, are bounded by a
function\footnote{The existence of such function can be translated into finding a Lyapunov function
for the corresponding SDEs, an important topic in PDE literatures \cite{Giesl:book07}. See Assumption~4.1 
in \cite{VollmerZT:arxiv15} and Appendix~\ref{sec:ass} for more details.}  $\mathcal{V}$, {\it i.e.}, 
$\|\mathcal{D}^k \psi\| \leq C_k\mathcal{V}^{p_k}$ for $k=(0, 1, 2, 3)$, $C_k, p_k > 0$. Furthermore, 
the expectation of $\mathcal{V}$ on $\{\Xb_{lh}\}$ is bounded: $\sup_l \mathbb{E}\mathcal{V}^p(\X_{lh}) < \infty$, 
and $\mathcal{V}$ is smooth such that 
$\sup_{s \in (0, 1)} \mathcal{V}^p\left(s\Xb + \left(1-s\right)\Yb\right) \leq C\left(\mathcal{V}^p\left(\Xb\right) + \mathcal{V}^p\left(\Yb\right)\right)$, $\forall \Xb, \Yb, p \leq \max\{2p_k\}$ for some $C > 0$.
\end{assumption}

We emphasize that our proof techniques are related to those of the SGLD \cite{MattinglyST:JNA10,VollmerZT:arxiv15}, 
but with significant distinctions in that, instead of expanding the function $\psi(\Xb_{lh})$ \cite{VollmerZT:arxiv15},
whose parameter $\Xb_{lh}$ does not endow an explicit form in general SG-MCMCs, we start from expanding 
the Kolmogorov's backward equation for each minibatch. Moreover, our techniques apply for general 
SG-MCMCs, instead of for one specific algorithm. More specifically, given a $K$th-order 
local integrator with the corresponding Kolmogorov operator $\tilde{P}_h^l$, according to 
Definition~\ref{def:k-order-integrator} and \eqref{eq:sgmcmc_integrator}, the Kolmogorov's backward 
equation for the $l$-th minibatch can be expanded as:
\begin{align}\label{eq:expandsion}
	\mathbb{E}[\psi(\Xb_{lh})] &= \tilde{P}_h^l \psi(\Xb_{(l-1)h}) = e^{h\tilde{\Lcal}_l} \psi(\Xb_{(l-1)h}) + O(h^{K+1}) \nonumber\\
	&= \left(\mathbb{I} + h\tilde{\Lcal}_l\right) \psi(\Xb_{(l-1)h}) + \sum_{k=2}^K\frac{h^k}{k!}\tilde{\Lcal}_l^k\psi(\Xb_{(l-1)h}) + O(h^{K+1})~,
\end{align}
where $\mathbb{I}$ is the identity map. Recall that $\tilde{\Lcal}_l = \Lcal + \Delta V_l$, {\it e.g.}, 
$\Delta V_l = (\nabla_{\theta}\tilde{U}_l - \nabla_{\theta} U) \cdot \nabla_{p}$ in SGHMC. By further using the 
Poisson equation \eqref{eq:PoissonEq1} to simplify related terms associated with $\Lcal$, after 
some algebra shown in Appendix~\ref{sec:bias_proof},  the bias can be derived from \eqref{eq:expandsion} as:
$|\mathbb{E}\hat{\phi} - \bar{\phi}| = $
\begin{align*}
	\left|\frac{\mathbb{E}[\psi(\Xb_{lh})] - \psi(\Xb_0)}{Lh}
	- \frac{1}{L}\sum_l \mathbb{E}[\Delta V_l\psi(\Xb_{(l-1)h})] - \sum_{k=2}^K\frac{h^{k-1}}{k!L}\sum_{l=1}^L \EE[\tilde{\Lcal}_l^k\psi(\Xb_{(l-1)h})]\right| + O(h^K)~.
\end{align*}

All terms in the above equation can be bounded, with details provided in Appendix~\ref{sec:bias_proof}.
This gives us a bound for the bias of an SG-MCMC algorithm in Theorem~\ref{theo:bias1}.

\begin{theorem}\label{theo:bias1}
	Under Assumption~\ref{ass:assumption1}, let $\left\|\cdot\right\|$ be the operator norm.
	The bias of an SG-MCMC with a $K$th-order integrator at time 
	$T = hL$ can be bounded as:
	\begin{align*}
		\left|\mathbb{E}\hat{\phi} - \bar{\phi}\right| = O\left(\frac{1}{Lh} + \frac{\sum_l \left\|\mathbb{E}\Delta V_l\right\|}{L} + h^K\right)~. 
	\end{align*}
\end{theorem}\vspace{-0.1cm}

Note the bound above includes the term $\sum_l \left\|\mathbb{E}\Delta V_l\right\| / L$, measuring 
the difference between the expectation of stochastic gradients and the true gradient. It vanishes 
when the stochastic gradient is an unbiased estimation of the exact gradient, an assumption made 
in the SGLD. This on the other hand indicates that if the stochastic gradient is biased,
$|\mathbb{E}\hat{\phi} - \bar{\phi}|$ might diverge when the growth of $\sum_l \|\mathbb{E} \Delta V_l\|$ 
is faster than $O(L)$. We point this out to show our result to be more informative than that of the 
SGLD \cite{VollmerZT:arxiv15}, though this case might not happen in real applications. By expanding 
the proof for the bias, we are also able to bound the MSE of SG-MCMC algorithms, given in 
Theorem~\ref{theo:MSE}.

\begin{theorem}\label{theo:MSE}
	Under Assumption~\ref{ass:assumption1}, and assume $\tilde{U}_l$ is an unbiased estimate 
	of $U_l$. For a smooth test function 
	$\phi$, the MSE of an SG-MCMC with a $K$th-order integrator at time $T = hL$ is bounded,
	for some $C > 0$ independent of $(L, h)$, as
	\vspace{-0.1cm}\begin{align*}
		\mathbb{E}\left(\hat{\phi} - \bar{\phi}\right)^2 \leq C \left(\frac{\frac{1}{L}\sum_l\mathbb{E}\left\|\Delta V_l\right\|^2}{L} + \frac{1}{Lh} + h^{2K}\right)~.
	\end{align*}
\end{theorem}\vspace{-0.1cm}

Compared to the SGLD \cite{VollmerZT:arxiv15}, the extra term 
$\frac{1}{L^2}\sum_l\mathbb{E}\left\|\Delta V_l\right\|^2$ 
relates to the variance of noisy gradients. As long as the variance is 
bounded, the MSE still converges with the same rate. Specifically, when 
optimizing bounds for the bias and MSE, the optimal bias decreases at a rate of $L^{-K/(K+1)}$ 
with step size $h \propto L^{-1/(K+1)}$; while this is $L^{-2K/(2K+1)}$ with step size 
$h \propto L^{-1/(2K+1)}$ for the MSE\footnote{To compare with the standard MCMC convergence 
rate of $1/2$, the rate needs to be taken a square root.}. These rates decrease faster than those 
of the SGLD \cite{VollmerZT:arxiv15} when $K \geq 2$. The case of $K = 2$ for the SGHMC with 
our proposed symmetric splitting integrator is discussed in Section~\ref{sec:integrators}.

\subsection{Stationary invariant measures}

The asymptotic invariant measures of SG-MCMCs correspond to $L$ approaching infinity in the above analysis.
According to the bias and MSE above, asymptotically ($L \rightarrow \infty$) the sample average 
$\hat{\phi}$ is a random variable with mean $\mathbb{E}\hat{\phi} = \bar{\phi} + O(h^K)$, and variance 
$\mathbb{E}(\hat{\phi} - \mathbb{E}\hat{\phi})^2 \leq \mathbb{E}(\hat{\phi} - \bar{\phi})^2 
+ \mathbb{E}(\bar{\phi} - \mathbb{E}\hat{\phi})^2 = O(h^{2K})$, close to the true $\bar{\phi}$.
This section defines distance between measures, and studies more formally how the approximation 
errors affect the invariant measures of SG-MCMC algorithms.

First we note that under mild conditions, the existence of a stationary invariant 
measure for an SG-MCMC can be guaranteed by application of the Krylov--Bogolyubov 
Theorem~\cite{BogoliubovK:AM37}. Examining the conditions is beyond the scope of this paper. 
For simplicity, we follow \cite{MattinglyST:JNA10} and assume stationary invariant measures 
do exist for SG-MCMCs. We denote the corresponding invariant measure as $\tilde{\rho}_h$, 
and the true posterior of a model as $\rho$. Similar to \cite{MattinglyST:JNA10}, we assume 
our numerical solver is geometric ergodic, meaning that for a test function $\phi$, we have
$\int_{\mathcal{X}} \phi(\xb) \tilde{\rho}_h(\mathrm{d}\xb) = \int_{\mathcal{X}} \mathbb{E}_{\xb} \phi(\Xb_{lh}) \tilde{\rho}_h(\mathrm{d}\xb)$
 for any $l \geq 0$ from the ergodic theorem, where $\mathbb{E}_{\xb}$ denotes the 
 expectation conditional on $\Xb_0 = \xb$. The geometric ergodicity implies that the integration 
 is independent of the starting point of an algorithm. Given this, we have the following theorem on 
 invariant measures of SG-MCMCs.

\begin{theorem}\label{theo:invariantmeasure}
Assume that a $K$th-order integrator is geometric ergodic and its invariance measures $\tilde{\rho}_h$ exist.
Define the distance between the invariant measures $\tilde{\rho}_h$ and $\rho$ as:
$d(\tilde{\rho}_h, \rho) \triangleq \sup_{\phi}\left|\int_{\mathcal{X}} \phi(\xb) \tilde{\rho}_h(\mathrm{d}\xb) - \int_{\mathcal{X}} \phi(\xb) \rho(\mathrm{d}\xb)\right|$. Then any invariant measure $\tilde{\rho}_h$ of an SG-MCMC is close to $\rho$ 
with an error up to an order of $O(h^K)$, {\it i.e.}, there exists some $C \geq 0$ such that:
$d(\tilde{\rho}_h, \rho) \leq C h^K$.
\end{theorem}

For a $K$th-order integrator with full gradients, the corresponding invariant measure has been 
shown to be bounded by an order of $O(h^K)$ \cite{MattinglyST:JNA10,LeimkuhlerS:arxiv15}. 
As a result, Theorem~\ref{theo:invariantmeasure} suggests only orders of numerical approximations
but not the stochastic gradient approximation affect the asymptotic invariant measure of an 
SG-MCMC algorithm. This is also reflected by experiments presented in Section~\ref{sec:exp}.

\subsection{SG-MCMCs with decreasing step sizes}
 
The original SGLD was first proposed with a decreasing-step-size
sequence \cite{WellingT:ICML11}, instead of fixing step sizes, as analyzed in \cite{VollmerZT:arxiv15}. 
In \cite{TehTV:arxiv14}, the authors provide theoretical foundations on its asymptotic convergence 
properties. We demonstrate in this section that for general SG-MCMC algorithms, decreasing 
step sizes for each minibatch are also feasible. Note our techniques here are different from those 
used for the decreasing-step-size SGLD \cite{TehTV:arxiv14}, which interestingly result in similar 
convergence patterns. Specifically, by adapting the same techniques used in the previous sections, 
we establish conditions on the step size sequence to ensure asymptotic convergence, and develop 
theory on their finite-time ergodic error as well. To guarantee asymptotic consistency, the following 
conditions on decreasing step size sequences are required.

\begin{assumption}\label{ass:assump3}
The step sizes $\{h_l\}$ are decreasing\footnote{Actually the sequence need not be 
decreasing; we assume it is decreasing for simplicity.}, {\it i.e.}, $0 < h_{l+1} < h_{l}$, and satisfy that
1) $\sum_{l=1}^\infty h_l = \infty$; and 2) $\lim_{L \rightarrow \infty}\frac{\sum_{l=1}^L h_l^{K+1}}{\sum_{l=1}^L h_l} = 0$.
\end{assumption}

Denote the finite sum of step sizes as $S_L \triangleq \sum_{l=1}^L h_l$. Under 
Assumption~\ref{ass:assump3}, we need to modify the sample average $\bar{\phi}$ 
defined in Section~\ref{sec:finitetimeerr} as a weighted summation of
$\{\phi(\Xb_{lh})\}$: $\tilde{\phi} = \sum_{l = 1}^L \frac{h_l}{S_{L}} \phi(\Xb_{lh})$. For simplicity, 
we assume $\tilde{U}_l$ to be an unbiased estimate of $U$ such that $\mathbb{E}\Delta V_l = 0$.
Extending techniques in previous sections, we develop the following bounds for the bias and MSE.

\begin{theorem}\label{theo:bias_w}
	Under Assumptions~\ref{ass:assumption1} and \ref{ass:assump3}, for a smooth test function 
	$\phi$, the bias and MSE of a decreasing-step-size SG-MCMC with a $K$th-order integrator at time 
	$S_L$ are bounded as:
	\begin{align}
		&\text{BIAS: }\left|\mathbb{E}\tilde{\phi} - \bar{\phi}\right| = O\left(\frac{1}{S_L} + \frac{\sum_{l=1}^L h_l^{K+1}}{S_L}\right) \label{eq:bias_decease}\\
		&\text{MSE: }\mathbb{E}\left(\tilde{\phi} - \bar{\phi}\right)^2 \leq C\left(\sum_l \frac{h_l^2}{S_L^2}\mathbb{E}\left\|\Delta V_l\right\|^2 + \frac{1}{S_L} + \frac{(\sum_{l=1}^L h_l^{K+1})^2}{S_L^2} \right)~. \label{eq:mse_decrease}
	\end{align}
	As a result, the asymptotic bias approaches 0 according to the assumptions. If further 
	assuming\footnote{The assumption of $\sum_{l=1}^{\infty}h_l^2 < \infty$ satisfies this requirement, 
	but is weaker than the original assumption.} $\frac{\sum_{l=1}^{\infty} h_l^2}{S_L^2} = 0$,
	the MSE also goes to 0. In words, the decreasing-step-size SG-MCMCs are consistent.
\end{theorem}

Among the kinds of decreasing step size sequences, a commonly recognized one is 
$h_l \propto l^{-\alpha}$ for $0 < \alpha < 1$. We show in the following corollary that such 
a sequence leads to a valid sequence.

\begin{corollary}\label{coro:stepsize}
Using the step size sequences $h_l \propto l^{-\alpha}$ for $0 < \alpha < 1$, all the step size 
assumptions in Theorem~\ref{theo:bias_w} are satisfied. As a result, the bias and MSE approach
zero asymptotically, {\it i.e.}, the sample average $\tilde{\phi}$ is asymptotically consistent with the 
posterior average $\bar{\phi}$.
\end{corollary}

\begin{remark}\label{remark:decreasingbiasmse}
Theorem~\ref{theo:bias_w} indicates the sample average $\tilde{\phi}$ asymptotically 
converges to the true posterior average $\bar{\phi}$.
It is possible to find out the optimal decreasing rates for the specific decreasing sequence 
$h_l \propto l^{-\alpha}$. Specifically, using the bounds for $\sum_{l=1}^L l^{-\alpha}$ 
(see the proof of Corollary~\ref{coro:stepsize}), for the two terms in the bias \eqref{eq:bias_decease} 
in Theorem~\ref{theo:bias_w}, $\frac{1}{S_L}$ decreases at a rate of $O(L^{\alpha-1})$, whereas 
$(\sum_{l=1}^L h_l^{K+1})/S_L$ decreases as $O(L^{-K\alpha})$. The balance between these 
two terms is achieved when $\alpha = 1/(K+1)$, which agrees with Theorem~\ref{theo:bias1} on 
the optimal rate of fixed-step-size SG-MCMCs. Similarly, for the MSE \eqref{eq:mse_decrease}, 
the first term decreases as $L^{-1}$, independent of $\alpha$, while the second and third terms 
decrease as $O(L^{\alpha-1})$ and $O(L^{-2K\alpha})$, respectively, thus the balance is achieved 
when $\alpha = 1/(2K+1)$, which also agrees with the optimal rate for the fixed-step-size MSE in 
Theorem~\ref{theo:MSE}. 
\vspace{-0.1cm}\end{remark}

According to Theorem~\ref{theo:bias_w}, one theoretical advantage of decreasing-step-size 
SG-MCMCs over fixed-step-size variants is the asymptotically unbiased estimation of posterior 
averages, though the benefit  might not be significant in large-scale real applications where the asymptotic 
regime is not reached.

\vspace{-0.1cm}\section{Practical Numerical Integrators}\label{sec:integrators}\vspace{-0.2cm}

Given the theory for SG-MCMCs with high-order integrators, we here propose a 2nd-order 
symmetric splitting integrator for practical use. The Euler integrator is known as a 1st-order 
integrator; the proof and its detailed applications on the SGLD and SGHMC are given in 
Appendix~\ref{sec:numerical_integrator}.

The main idea of the symmetric splitting scheme is to split the local generator $\tilde{\Lcal}_l$ 
into several sub-generators that can be solved analytically\footnote{This is different from the 
traditional splitting in SDE literatures\cite{LeimkuhlerM:AMRE13,LeimkuhlerS:arxiv15}, where 
$\Lcal$ instead of $\tilde{\Lcal}_l$ is split.}. Unfortunately, one cannot easily apply a splitting 
scheme with the SGLD. However, for the SGHMC, it can be readily split into: 
$\tilde{\Lcal}_l = \mathcal{L}_A  + \mathcal{L}_B + \mathcal{L}_{O_l}$, where
\begin{align}\label{eq:split_generator}
	\mathcal{L}_A = \pb \cdot \nabla_{\thetab}, \;\;
	\mathcal{L}_B = -D \pb \cdot \nabla_{\pb}, \;\;
	\mathcal{L}_{O_l} = -\nabla_{\thetab}\tilde{U}(\thetab) \cdot \nabla_{\pb} + 2D \Ib_n:\nabla_{\pb} \nabla_{\pb}^T~.
\end{align}
These sub-generators correspond to the following SDEs, which are all analytically solvable:
\begin{align}\label{eq:split_sghmc}
	\hspace{-0.2cm}A: \left\{\begin{array}{ll}
	\mathrm{d}\thetab &= \pb \mathrm{d}t \\
	\mathrm{d}\pb &= 0
	\end{array}\right.,
	B: \left\{\begin{array}{ll}
	\mathrm{d}\thetab &= 0 \\
	\mathrm{d}\pb &= - D \pb \mathrm{d}t
	\end{array}\right.,
	O: \left\{\begin{array}{ll}
	\mathrm{d}\thetab &= 0 \\
	\mathrm{d}\pb &= -\nabla_\thetab \tilde{U}_l(\thetab) \mathrm{d}t + \sqrt{2D}\mathrm{d}W
	\end{array}\right.
\end{align}
Based on these sub-SDEs, the local Kolmogorov operator $\tilde{P}_{h}^{l}$ is defined as:
\begin{align*}
\EE[f(\Xb_{lh})] = \tilde{P}_{h}^{l}f(\Xb_{(l-1)h}), \;\; \text{where, } \; \tilde{P}_{h}^{l} \triangleq e^{\frac{h}{2}\mathcal{L}_A} \circ e^{\frac{h}{2}\mathcal{L}_B} \circ e^{h\mathcal{L}_{O_l}} \circ e^{\frac{h}{2}\mathcal{L}_B} \circ e^{\frac{h}{2}\mathcal{L}_A}, 
\end{align*}

so that the corresponding updates for $\Xb_{lh} = (\thetab_{lh}, \pb_{lh})$ consist of the 
following 5 steps:
\begin{align*}
	&\thetab_{lh}^{(1)} = \thetab_{(l-1)h} + \pb_{(l-1)h} h / 2 \Rightarrow \pb_{lh}^{(1)} = e^{-Dh/2} \pb_{(l-1)h} \Rightarrow \pb_{lh}^{(2)} = \pb_{lh}^{(1)} - \nabla_{\thetab}\tilde{U}_l(\thetab_{lh}^{(1)}) h + \sqrt{2Dh} \zetab_l \\
	&\Rightarrow \hspace{0.2cm} \pb_{lh} = e^{-Dh/2} \pb_{lh}^{(2)} \hspace{0.2cm} \Rightarrow \hspace{0.2cm} \thetab_{lh} = \thetab_{lh}^{(1)} + \pb_{lh} h / 2~,
\end{align*}
where $(\thetab_{lh}^{(1)}, \pb_{lh}^{(1)}, \pb_{lh}^{(2)})$ are intermediate variables.
We denote such a splitting method as the ABOBA scheme. From the Markovian property 
of a Kolmogorov operator, it is readily seen that all such symmetric splitting schemes 
(with different orders of `A', `B' and `O') are equivalent \cite{LeimkuhlerM:AMRE13}. 
Lemma~\ref{lem:splitting} below shows the symmetric splitting scheme is a 2nd-order 
local integrator.

\begin{lemma}\label{lem:splitting}
The symmetric splitting scheme is a 2nd-order local integrator, {\it i.e.}, the corresponding
Kolmogorov operator $\tilde{P}_{h}^{l}$ satisfies:
$\tilde{P}_{h}^{l} = e^{h\tilde{\mathcal{L}}_l} + O(h^3)$.
\end{lemma}

When this integrator is applied to the SGHMC, the following properties can be obtained.

\begin{remark}\label{remark:bias}
	Applying Theorem~\ref{theo:bias1} to the SGHMC with the symmetric splitting scheme ($K = 2$), 
	the bias is bounded as:
	$|\mathbb{E}\hat{\phi} - \bar{\phi}| = O(\frac{1}{Lh} + \frac{\sum_l \left\|\mathbb{E}\Delta V_l\right\|}{L} + h^2)$. 
	The optimal bias decreasing rate is $L^{-2/3}$, compared to $L^{-1/2}$ for 
	the SGLD \cite{VollmerZT:arxiv15}. Similarly, the MSE is bounded by: 
	$\mathbb{E}(\hat{\phi} - \bar{\phi})^2 \leq C (\frac{\frac{1}{L}\sum_l\mathbb{E}\left\|\Delta V_l\right\|^2}{L} + \frac{1}{Lh} + h^{4})$, 
	decreasing optimally as $L^{-4/5}$ with step size $h \propto L^{-1/5}$, compared to the MSE 
	of $L^{-2/3}$ for the SGLD \cite{VollmerZT:arxiv15}. This indicates that the SGHMC with the 
	splitting integrator converges faster than the SGLD and SGHMC with 1st-order Euler integrators.
\end{remark}

\begin{remark}
For a decreasing-step-size SGHMC, based on Remark~\ref{remark:decreasingbiasmse}, 
the optimal step size decreasing rate for the bias 
is $\alpha = 1/3$, and $\alpha = 1/5$ for the MSE. These agree with their fixed-step-size 
counterparts in Remark~\ref{remark:bias},
thus are faster than the SGLD/SGHMC with 1st-order Euler integrators.
\end{remark}

\vspace{-0.1cm}\section{Experiments}\label{sec:exp}\vspace{-0.2cm}

We here verify our theory and compare with related algorithms on both synthetic data and 
large-scale machine learning applications.

\vspace{-0.4cm}\paragraph{Synthetic data}

\begin{wrapfigure}{R}{5cm}\vspace{-25pt}
  \centering
  \includegraphics[width=\linewidth]{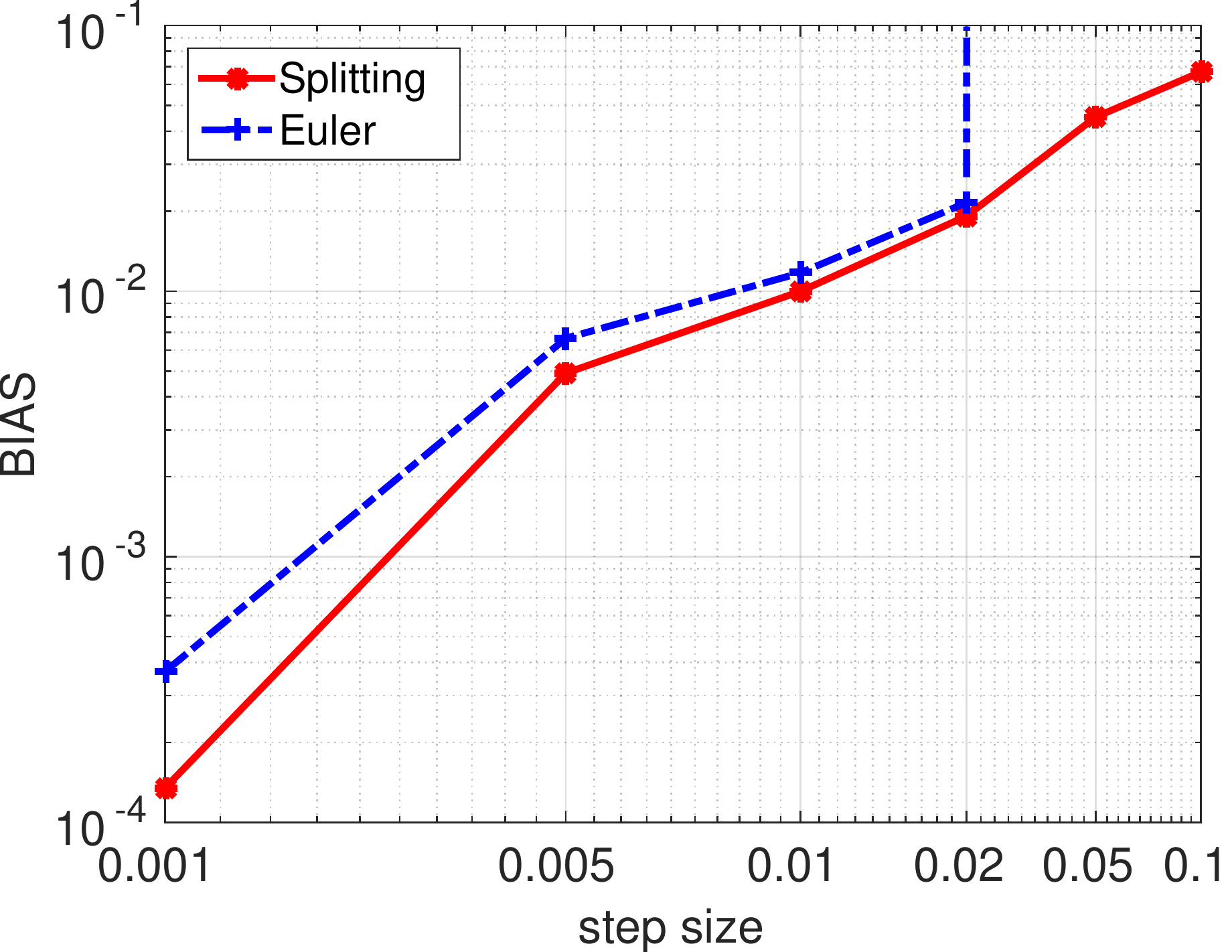}
  \caption{Comparisons of symmetric splitting and Euler integrators.}\label{fig:split_euler}\vspace{-20pt}
\end{wrapfigure}

We consider a standard Gaussian model where $x_i \sim \mathcal{N}(\theta, 1), \theta \sim \mathcal{N}(0, 1)$. 1000 
data samples $\{x_i\}$ are generated,
and every minibatch in the stochastic gradient is of size 10. The test function is defined as 
$\phi(\theta) \triangleq \theta^2$, with explicit expression for the posterior average.
To evaluate the expectations in the bias and MSE, we average over 200 runs with random initializations.

\vspace{-0.1cm}First we compare the invariant measures (with $L = 10^6$) of the proposed splitting integrator and Euler 
integrator for the SGHMC. Results of the SGLD are omitted since they are not as competitive. 
Figure~\ref{fig:split_euler} plots the biases with different step sizes. It is clear that the Euler 
integrator has larger biases in the invariant measure, and quickly explodes when the step size becomes 
large, which does not happen for the splitting integrator. In real applications we also find this happen 
frequently (shown in the next section), making the Euler scheme an 
unstable integrator.

\vspace{-0.1cm}Next we examine the asymptotically optimal step size rates for the SGHMC. From the theory 
we know these are $\alpha = 1/3$ for the bias and $\alpha = 1/5$ for the MSE, in both 
fixed-step-size SGHMC (SGHMC-F) and decreasing-step-size SGHMC (SGHMC-D). For the 
step sizes, we did a grid search to select the best prefactors, which resulted in 
$h\!\!=\!\!0.033\!\times\!L^{-\alpha}$ for the SGHMC-F and $h_l\!\!=\!0.045\!\!\times\!l^{-\alpha}$ 
for the SGHMC-D, with different $\alpha$ values. We plot the traces of the bias for the SGHMC-D 
and the MSE for the SGHMC-F in Figure~\ref{fig:bias_mse_gau_rates}. Similar results for the bias 
of the SGHMC-F and the MSE of the SGHMC-D are plotted in Appendix~\ref{sec:extra_exp}. We find 
that when rates are smaller than the theoretically optimal rates, {\it i.e.}, $\alpha = 1/3$ (bias) 
and $\alpha = 1/5$ (MSE), the bias and MSE tend to decrease faster than the optimal rates 
at the beginning (especially for the SGHMC-F), but eventually they slow down and are surpassed 
by the optimal rates, consistent with the {\em asymptotic} theory. This also suggests that
if only a small number of iterations were feasible, setting a larger step size than the theoretically 
optimal one might be beneficial in practice.

\begin{figure}[t!]
\vskip -0.1in
\centering
\begin{minipage}{0.43\linewidth}
  \includegraphics[width=0.9\linewidth]{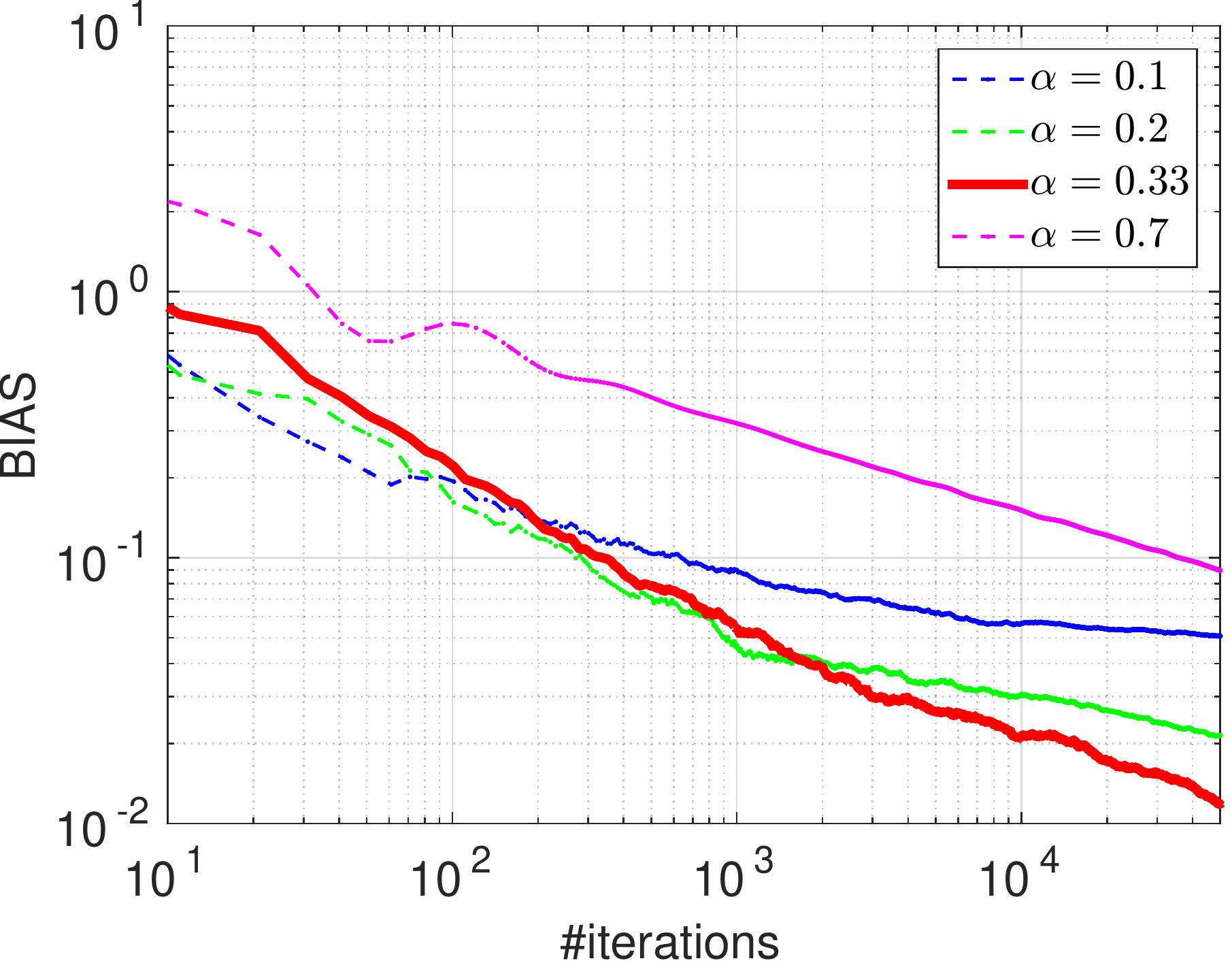}
\end{minipage}
\begin{minipage}{0.43\linewidth}
  \includegraphics[width=0.9\linewidth]{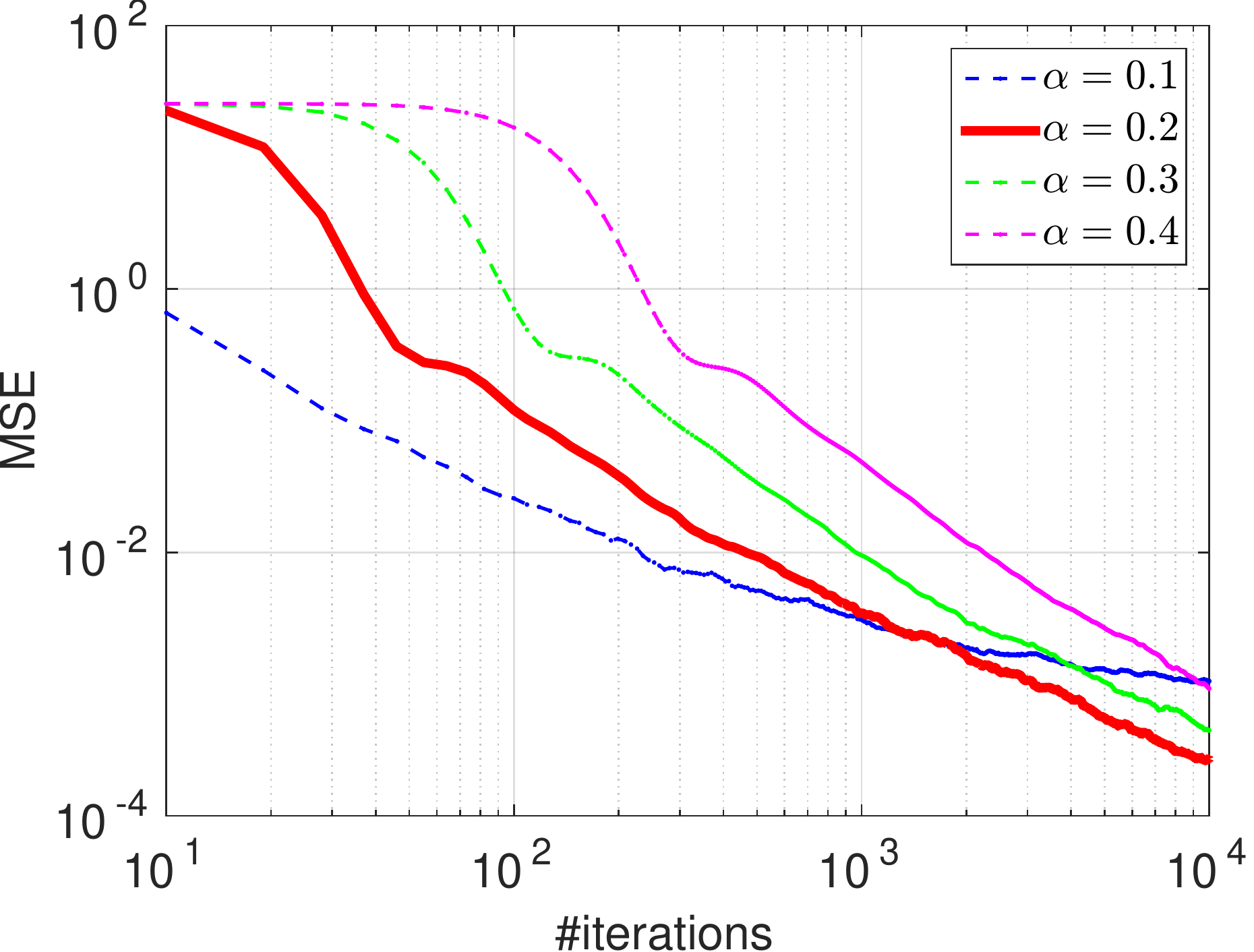}
\end{minipage}
\vskip -0.1in
\caption{{\em Bias} of SGHMC-D (left) and {\em MSE} of SGHMC-F (right) with different step size 
rates $\alpha$. Thick red curves correspond to theoretically optimal rates.}
\label{fig:bias_mse_gau_rates}
\vskip -0.2in
\end{figure}

\vspace{-0.1cm}Finally, we study the relative convergence speed of the SGHMC and SGLD.
We test both fixed-step-size and decreasing-step-size versions.
For fixed-step-size experiments, the step sizes are set to $h = C L^{-\alpha}$, with $\alpha$ chosen
according to the theory for SGLD and SGHMC. To provide a fair comparison, the constants $C$ are 
selected via a grid search from $10^{-3}$ to 0.5 with an interval of $0.002$ for $L = 500$, it is then 
fixed in the other runs with different $L$ values. The parameter $D$ in the SGHMC is selected within 
$(10, 20, 30)$ as well. For decreasing-step-size experiments, an initial step size is chosen within 
$[0.003, 0.05]$ with an interval of $0.002$ for different algorithms\footnote{Using the same initial 
step size is not fair because the SGLD requires much smaller step sizes.}, and then it decreases 
according to their theoretical optimal rates. Figure~\ref{fig:bias_mse_gau} shows a comparison of 
the biases for the SGHMC and SGLD. As indicated by both theory and experiments, the SGHMC with the splitting integrator 
yields a faster convergence speed than the SGLD with an Euler integrator.

\begin{figure}[t!]
\centering
\begin{minipage}{0.43\linewidth}
  \includegraphics[width=\linewidth]{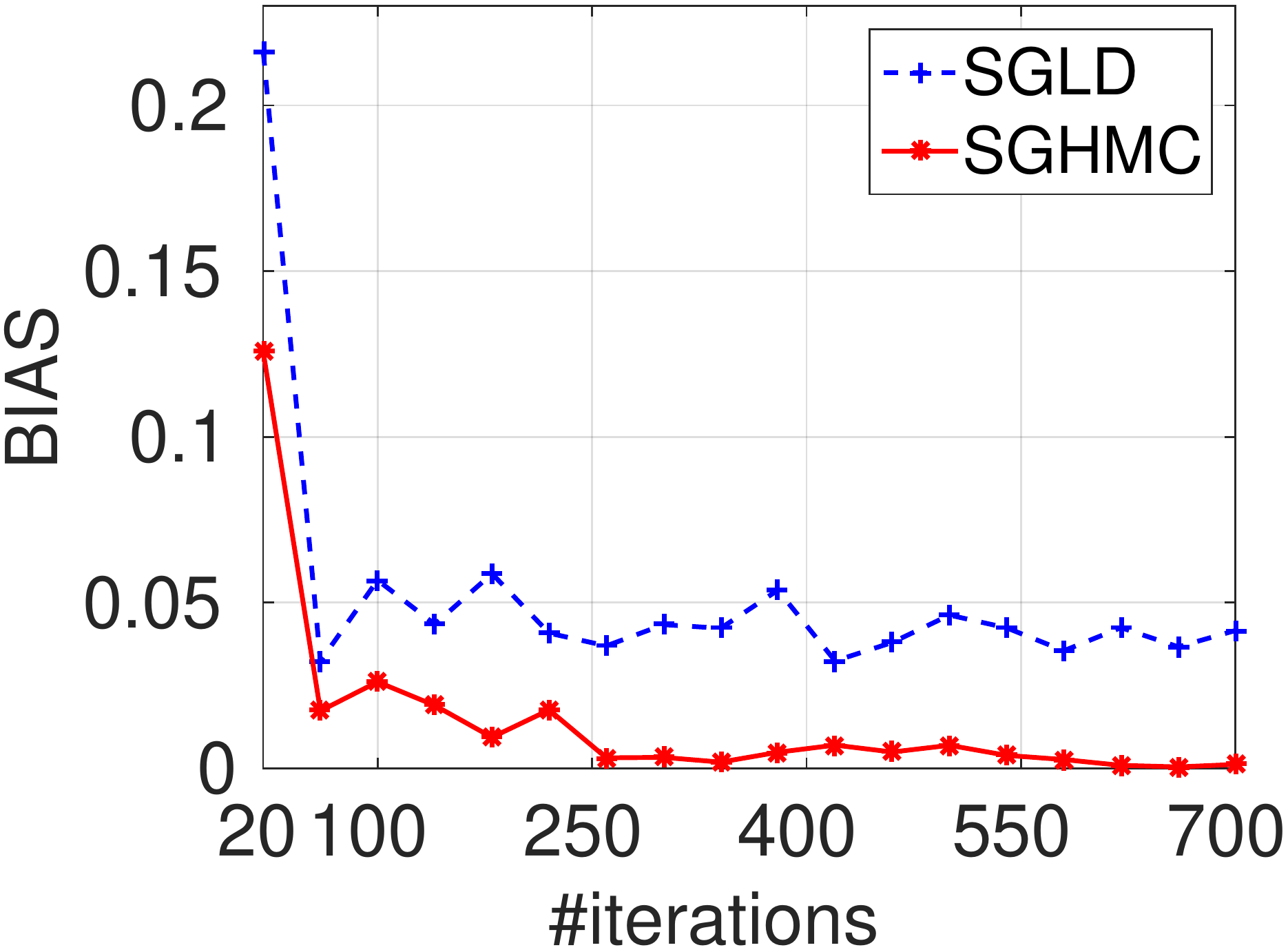}
\end{minipage}
\begin{minipage}{0.43\linewidth}
  \includegraphics[width=\linewidth]{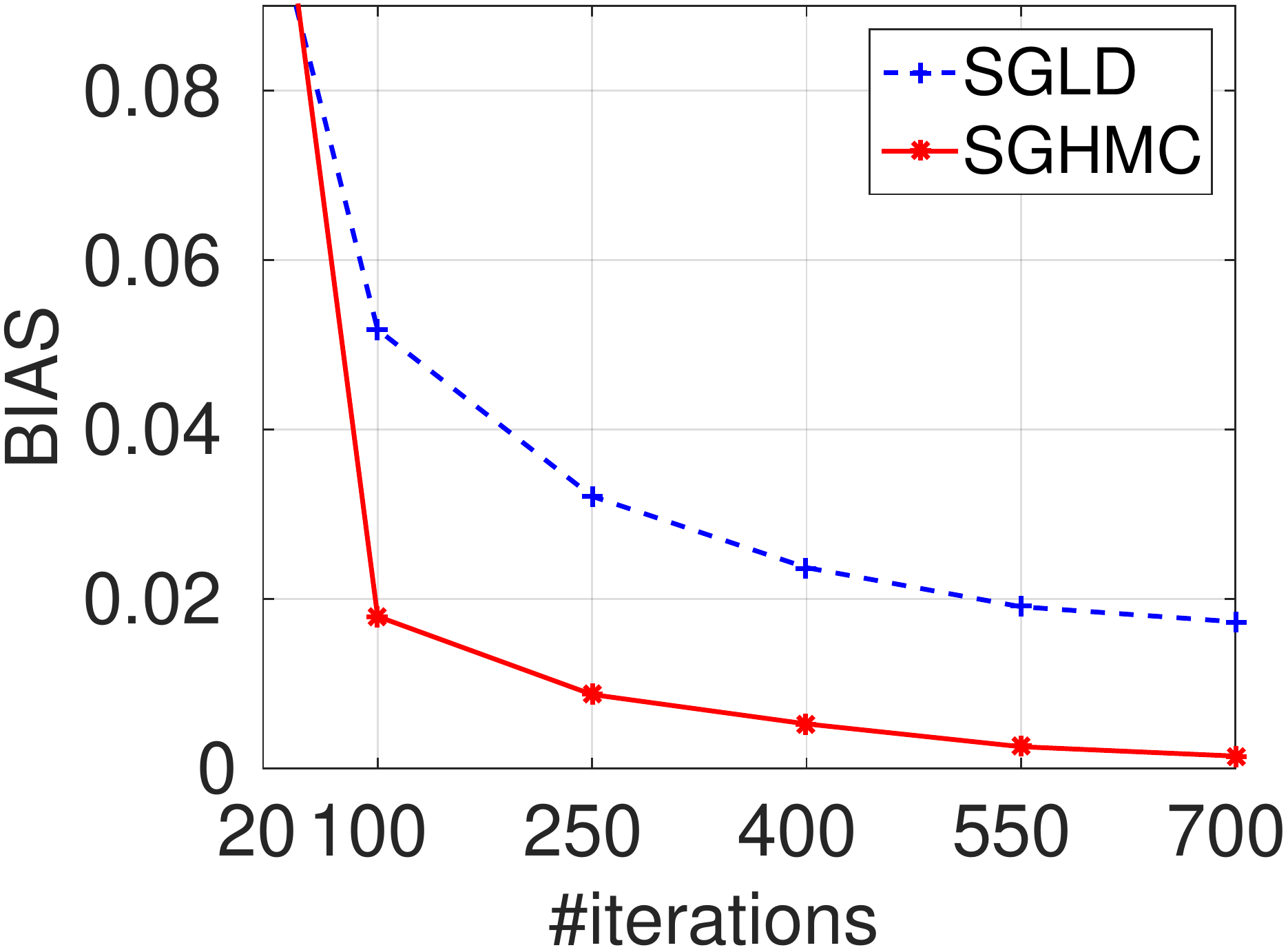}
\end{minipage}
\vskip -0.15in
\caption{{\em Biases} for the fixed-step-size (left) and decreasing-step-size (right)
SGHMC and SGLD.}
\label{fig:bias_mse_gau}
\vskip -0.25in
\end{figure}

\vspace{-0.3cm}\paragraph{Large-scale machine learning applications}

For real applications, we test the SGLD with an Euler integrator, the SGHMC with the splitting integrator (SGHMC-S),
and the SGHMC with an Euler integrator (SGHMC-E). First we test them on the latent 
Dirichlet allocation model (LDA) \cite{blei2003latent}. The data used consists of 10M randomly 
downloaded documents from \emph{Wikipedia}, using scripts provided in \cite{HoffmanBB:NIPS10}. 
We randomly select 1K documents for testing and validation, respectively. As 
in \cite{HoffmanBB:NIPS10,GanCHCC:icml15}, the vocabulary size is 7,702. We use 
the {\em Expanded-Natural} reparametrization trick to sample from the probabilistic 
simplex \cite{PattersonT:NIPS13}. 
The step sizes are chosen from $\{2, 5, 8, 20, 50, 80\}\!\!\times\!\!10^{-5}$, 
and parameter $D$ from $\{20, 40, 80\}$. The minibatch size is set to 100, with one pass of the whole data
in the experiments (and therefore $L=100K$). We collect 300 posterior samples 
to calculate test perplexities, with a standard holdout technique as described in \cite{GanCHCC:icml15}.

\vspace{-0.1cm}Next a recently studied sigmoid belief network model (SBN) \cite{GanHCC:aistats15} is tested,
which is a directed counterpart of the popular RBM model. We use a one layer model where the 
bottom layer corresponds to binary observed data, which is generated from the hidden layer 
(also binary) via a sigmoid function. As shown in \cite{GanCHCC:icml15}, the SBN is readily 
learned by SG-MCMCs. We test the model on the MNIST dataset, which consists of 60K hand 
written digits of size $28\times 28$ for training, and 10K for testing. Again the step sizes are chosen 
from $\{3, 4, 5, 6\}\!\!\times\!\!10^{-4}$, $D$ from $\{0.9, 1, 5\}/\sqrt{h}$. The minibatch is set to 
200, with 5000 iterations for training. Like applied for the  RBM \cite{SalakhutdinovM:icml08}, an advance 
technique called anneal importance sampler (AIS) is adopted for calculating test likelihoods.

\begin{wrapfigure}{R}{6.3cm}\vspace{-20pt}
  \centering
  \includegraphics[width=\linewidth]{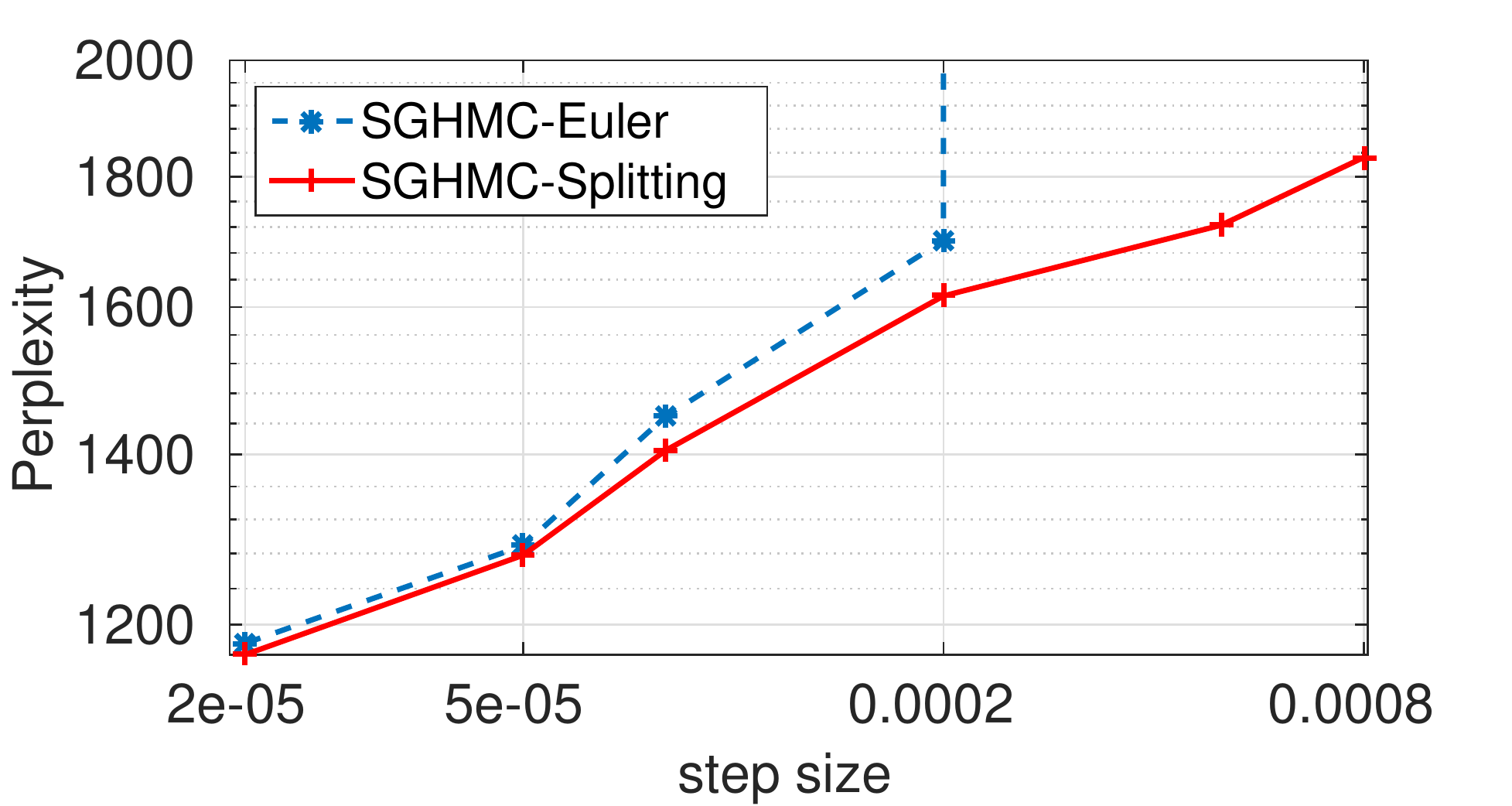}
  \vspace{-16pt}\caption{SGHMC with 200 topics. 
  The Euler explodes with large step sizes.}\label{fig:lda_split_euler}\vspace{-16.5pt}
\end{wrapfigure}

We briefly describe the results here, more details are provided in Appendix~\ref{sec:extra_exp}. 
For LDA with 200 topics, the best test perplexities for the SGHMC-S, SGHMC-E
and SGLD are 1168, 1180 and 2496, respectively; while these are 1157, 1187 and 2511, 
respectively,  for 500 topics.
Similar to the synthetic experiments, we also observed SGHMC-E crashed 
when using large step sizes. This is illustrated more clearly in Figure~\ref{fig:lda_split_euler}.
For the SBN with 100 hidden units, we obtain negative test log-likelihoods of 103, 105 and
126 for the SGHMC-S, SGHMC-E and SGLD, respectively; and these are 98, 100, and 110 for 200 hidden
units. Note the SGHMC-S on SBN yields state-of-the-art results on test likelihoods compared 
to \cite{MnihG:icml14}, which was 113 for 200 hidden units. 
A decrease of 2 units in the neg-log-likelihood with AIS is considered to be a 
reasonable gain \cite{GanHCC:aistats15}, which is approximately equal to the gain from a 
shallow to a deep model \cite{MnihG:icml14}. 
SGHMC-S is more accuracy and robust than SGHMC-E due to its 2nd-order splitting integrator.



\vspace{-0.5cm}\section{Conclusion}\vspace{-0.4cm}

For the first time, we develop theory to analyze finite-time ergodic errors, as well as 
asymptotic invariant measures, of general SG-MCMCs with high-order integrators.
Our theory applies for both fixed and decreasing step size SG-MCMCs, which are 
shown to be equivalent in terms of convergence rates, and are faster with our proposed 
2nd-order integrator than previous SG-MCMCs with 1st-order Euler integrators. 
Experiments on both synthetic and large real datasets validate our theory. The theory also 
indicates that with increasing order of numerical integrators, the convergence rate of an SG-MCMC is 
able to theoretically approach the standard MCMC convergence rate. Given the theoretical 
convergence results, SG-MCMCs can be used effectively in real applications.

\vspace{-0.5cm}\paragraph{Acknowledgments}
Supported in part by ARO, DARPA, DOE, NGA and ONR.
We acknowledge Jonathan C. Mattingly and Chunyuan Li for inspiring discussions; David Carlson
for the AIS codes.

\bibliographystyle{unsrtnat}
\bibliography{references}

\newpage
\noindent\makebox[\linewidth]{\rule{\linewidth}{3.5pt}}
\begin{center}
\bf{\Large Supplementary Material for:\\ On the Convergence of Stochastic Gradient MCMC Algorithms with High-Order Integrators}
\end{center}
\noindent\makebox[\linewidth]{\rule{\linewidth}{1pt}}

\begin{center}
Changyou Chen$^\dag$~~~~~~~~~~~~ Nan Ding$^\ddag$ ~~~~~~~~~~~~ Lawrence Carin$^\dag$ \\
$^\dag$Dept. of Electrical and Computer Engineering, Duke University, Durham, NC, USA \\
$^\ddag$Google Inc., Venice, CA, USA \\
\texttt{cchangyou@gmail.com; dingnan@google.com; lcarin@duke.edu}
\end{center}

\appendix

\section{Representative Stochastic Gradient MCMC Algorithms}\label{sec:sgmcmc_review}

This section briefly introduces three recently proposed stochastic gradient MCMC algorithms,
including the stochastic gradient Langevin dynamic (SGLD) \cite{WellingT:ICML11}, the
stochastic gradient Hamiltonian MCMC (SGHMC) \cite{ChenFG:ICML14}, and the stochastic
gradient Nos\'{e}-Hoover thermostat \cite{DingFBCSN:NIPS14} (SGNHT).

Given data $\Xb = \{\xb_1, \cdots, \xb_N\}$, a generative model $p(\Xb | \thetab) = \prod_{i=1}^N p(\xb_i | \thetab)$
with model parameter $\thetab$, and prior $p(\thetab)$, we want to compute the posterior: 
$$\pi(\thetab) \triangleq p(\thetab | \Xb) \propto p(\Xb | \thetab) p(\thetab) \triangleq e^{-U(\thetab)}~.$$

\subsection{Stochastic gradient Langevin dynamics}
The SGLD \cite{WellingT:ICML11} is based on the following 1st-order Langevin dynamic defined as:
\begin{equation}\label{eq:BD}
	\mathrm{d}\thetab = -\frac{1}{2}\nabla_{\thetab} U(\thetab) \mathrm{d}t + \mathrm{d}W~,
\end{equation}
where $W$ is the standard Brownian motion.
We can show via the Fokker--Planck equation that the equilibrium distribution of \eqref{eq:BD} is:
$$p(\thetab) = \pi(\thetab)~.$$

As described in the main text, when sampling from this continuous-time diffusion, two approximations
are adopted, {\it e.g.}, a numerical integrator and a stochastic gradient version $\tilde{U}_l(\thetab_{(l-1)h})$ of 
the log-likelihood $U(\thetab)$ from the $l$-th minibatch. This results in the following SGLD algorithm.

\begin{algorithm}
\SetKwInOut{Input}{Input}
\caption{Stochastic Gradient Langevin Dynamics}
\Input{Parameters $h$.}
Initialize $\thetab_{0} \in \RR^n$ \;
\For {$l = 1, 2, \ldots $} {
Evaluate $\nabla \tilde{U}_l(\thetab_{(l-1)h})$ from the $l$-th minibatch \;
$\thetab_{lh} = \thetab_{(l-1)h}  - \nabla \tilde{U}_l(\thetab_{(l-1)h}) h + \sqrt{2h}\Ncal(0, I)$\;
}
\label{alg:sgld}
\end{algorithm}

\subsection{Stochastic gradient Hamiltonian MCMCs}

The SGHMC \cite{ChenFG:ICML14} is based on the 2nd-order Langevin dynamic defined as:
\begin{align}\label{eq:LD}
	\left\{\begin{array}{ll}
		\mathrm{d}\thetab &= \pb \mathrm{d}t \\
		\mathrm{d}\pb &= -\nabla_{\thetab} U(\thetab) \mathrm{d}t - D\pb \mathrm{d}t + \sqrt{2D} \mathrm{d}W~,
	\end{array} \right.
\end{align}
where $D$ is a constant independent of $\thetab$ and $\pb$. Again we can show that the equilibrium distribution
of \eqref{eq:LD} is:
$$P(\thetab, \pb) \propto e^{-U(\thetab) + \frac{\pb^T\pb}{2}}~.$$

Similar to the SGLD, we use the Euler scheme to simulate the dynamic \eqref{eq:LD}, shown in Algorithm~\ref{alg:sghmc}.

\begin{algorithm}
\SetKwInOut{Input}{Input}
\caption{Stochastic Gradient Hamiltonian MCMC}
\Input{Parameters $h, D$.}
Initialize $\thetab_{0} \in \RR^n$, $\pb_{0} \sim \Ncal(0,\Ib)$ \;
\For {$l = 1, 2, \ldots $} {
Evaluate $\nabla \tilde{U}_l(\thetab_{(l-1)h})$ from the $l$-th minibatch \;
$\pb_{lh} = \pb_{(l-1)h} - D \pb_{(l-1)h} h - \nabla \tilde{U}_l(\thetab_{(l-1)h}) h + \sqrt{2Dh}\Ncal(0, 1)$\;
$\thetab_{lh} = \thetab_{(l-1)h} + \pb_{lh} h$\;
}
\label{alg:sghmc}
\end{algorithm}

\subsection{Stochastic gradient No\'{s}e-Hoover thermostats}

The SGNHT \cite{DingFBCSN:NIPS14} is based on the No\'{s}e-Hoover thermostat defined as:
\begin{align}\label{eq:NHT}
	\left\{\begin{array}{ll}
		\mathrm{d}\thetab &= \pb \mathrm{d}t \\
		\mathrm{d}\pb &= -\nabla_{\thetab} U(\thetab) \mathrm{d}t - \xi \pb \mathrm{d}t + \sqrt{2D} \mathrm{d}W \\
		\mathrm{d}\xi &= \left(\pb^T\pb / n - 1\right)~,
	\end{array} \right.
\end{align}
If $D$ is independent of $\thetab$ and $\pb$, it can also be shown that the equilibrium distribution
of \eqref{eq:NHT} is \cite{DingFBCSN:NIPS14}:
\begin{align*}
	P(\thetab, \pb, \xi) \propto e^{-U(\thetab) - \frac{1}{2}\pb^T\pb + \frac{1}{2}(\xi - D)^2}~.
\end{align*}

The SGNHT is much more interesting than the SGHMC when considering subsampling data in each iteration,
as the covariance $D$ in SGHMC is hard to estimate, a thermostat is used to adaptively control the system 
temperature, thus automatically estimate the unknown $D$. The whole algorithm is shown in Algorithm~\ref{alg:sgnht}.

\begin{algorithm}
\SetKwInOut{Input}{Input}
\caption{Stochastic Gradient Nos\'{e}-Hoover Thermostats}
\Input{Parameters $h, D$.}
Initialize $\thetab_{0} \in \RR^n$, $\pb_{0} \sim \Ncal(0,\Ib)$, and $\xi_{0}=D$ \;
\For {$l = 1, 2, \ldots $} {
Evaluate $\nabla \tilde{U}_l(\thetab_{(l-1)h})$ from the $l$-th minibatch \;
$\pb_{lh} = \pb_{(l-1)h} - \xi_{(l-1)h} \pb_{(l-1)h} h - \nabla \tilde{U}_l(\thetab_{(l-1)h}) h + \sqrt{2Dh}\Ncal(0, I)$\;
$\thetab_{lh} = \thetab_{(l-1)h} + \pb_{lh} h$\;
$\xi_{lh} = \xi_{(l-1)h} + (\frac{1}{n} \pb_{lh}^{\top} \pb_{lh} - 1) h $\;
}
\label{alg:sgnht}
\end{algorithm}

\section{More Details on Kolmogorov's Backward Equation}\label{sec:KBE}

The generator $\mathcal{L}$ is used in the formulation of Kolmogorov's backward equation, 
which intuitively tells us how the expected
value of any suitably smooth statistic of $\Xb$ evolves in time. More precisely:
\begin{definition}[Kolmogorov's Backward Equation]
Let $u(t, \xb) = \mathbb{E}\left[\phi(\Xb_t)\right]$,
then $u(t, \xb)$ satisfies the following partial differential equation, known as {\em Kolmogorov's backward equation}:
\begin{align}\label{eq:koleq1}
	\left\{\begin{array}{ll}
		\frac{\partial u}{\partial t}(t, \xb) = \mathcal{L}u(t, \xb)~, & t > 0, \xb \in \mathbb{R}^n \\
		u(0, \xb) = \phi(\xb), & \xb \in \RR^n
		\end{array} \right.
\end{align}
\end{definition}

Based on the definition, we can write $u(t, \cdot) = P_t \phi(\cdot)$ so that $(P_t)_{t \geq 0}$ is the transition semigroup associated
with the Markov process $(\Xb(t, \xb))_{t\geq 0, \xb \in \mathbb{R}^n}$ \cite{DebusscheF:SIAMJNA12} (also
called the Kolmogorov operator).
Note that the {\em Kolmogorov's backward equation} can be written in another form as:
\begin{align}\label{eq:koleq}
	u(t, \xb) = \mathbb{E}\left[\phi(\Xb_t)\right] = e^{t\mathcal{L}}\phi(\xb)~,
\end{align}
where $e^{t\mathcal{L}}$ is the exponential map operator associated with the generator defined as:
\begin{align*}
	e^{t\mathcal{L}} \triangleq \mathbb{I} + \sum_{i=1}^\infty \frac{\left(t\mathcal{L}\right)^i}{i!}~,
\end{align*}
with $\mathbb{I}$ being the identity map. This is obtained by expanding $u(t, \xb)$ in time by using 
Taylor expansion \cite{DebusscheF:SIAMJNA12}:
\begin{align}\label{eq:kb_expansion}
	u(t, \xb) &= u(0, \xb) + \sum_{i=1}^\infty \frac{t^i}{i!}\frac{\mathrm{d}^i}{\mathrm{d}t^i} u(t, \xb) \left|_{t=0}\right. \nonumber\\
	&= u(0, \xb) + \sum_{i=1}^\infty \frac{t^i}{i!}\frac{\mathrm{d}^{i-1}}{\mathrm{d}t^{i-1}} \frac{\mathrm{d}}{\mathrm{d}t} u(t, \xb) \left|_{t=0}\right. \nonumber\\
	&= u(0, \xb) + \sum_{i=1}^\infty \frac{t^i}{i!} \mathcal{L}\frac{\mathrm{d}^{i-1}}{\mathrm{d}t^{i-1}} u(t, \xb) \left|_{t=0}\right. \nonumber\\
	&= \phi(\xb) + \sum_{i=1}^\infty \frac{t^i}{i!} \mathcal{L}^i \phi(\xb) = e^{t\mathcal{L}}\phi(\xb)~.
\end{align}

The form \eqref{eq:koleq} instead of the original form \eqref{eq:koleq1} of the Kolmogorov's backward equation
is used in our analysis. To be able to expand the form \eqref{eq:koleq} to some particular order such that remainder
terms are bounded, the following assumption is required \cite{Kopec:JNA14}.

\begin{assumption}
Assume 1) $F(\Xb)$ is $C^{\infty}$ with bounded derivatives of any order, furthermore, and 2) $|F(\xb)| \leq C (1 + |\Xb|^s)$
for some positive integer $s$. Under these assumptions, series of the generator expansion can be bounded, thus 
\eqref{eq:kb_expansion} can be written in the following form \cite{Kopec:JNA14,VollmerZT:arxiv15}:
\begin{align}
	u(t, \xb) = \phi(\xb) + \sum_{i=1}^\ell \frac{t^i}{i!} \mathcal{L}^i \phi(\xb) + t^{\ell + 1} r_{\ell}(F, \phi)(\xb)~,
\end{align}
with $\left|r_{\ell}(F, \phi)(\xb)\right| \leq C_{\ell}(1 + |\xb|^{k_{\ell}})$ for some constant $C_{\ell}, k_{\ell}$.
\end{assumption}

\section{More Comments on Assumption~\ref{ass:assumption1}}\label{sec:ass}

Assumption~\ref{ass:assumption1} assumes that the solution functional $\psi$ of the Poisson equation 
\eqref{eq:PoissonEq1} satisfies: $\psi$ and its up to 3-rd order derivatives, $\mathcal{D}^k \psi$, 
are bounded by a function $\mathcal{V}$, {\it i.e.}, $\|\mathcal{D}^k \psi\| \leq C_k\mathcal{V}^{p_k}$ for $k=(0, 1, 2, 3)$, 
$C_k, p_k > 0$. 
Furthermore, $\mathcal{V}$ is smooth such that $\sup_{s \in (0, 1)} \mathcal{V}^p(s\Xb + (1-s)\Yb) \leq C(\mathcal{V}^p(\Xb) + \mathcal{V}^p(\Yb))$, 
$\forall \Xb, \Yb, p \leq p^{*} \triangleq \max\{2p_k\}$ for some $C > 0$. Finally, $\sup_l \mathbb{E}\mathcal{V}^p(\X_{lh}) < \infty$ 
for $p \leq p^{*}$. This is summarized as:
\begin{align}
	\sup_l \mathbb{E}\mathcal{V}^p(\X_{lh}) &< \infty \label{eq:ass_eq1}\\
	\sup_{s \in (0, 1)} \mathcal{V}^p(s\Xb + (1-s)\Yb) &\leq C(\mathcal{V}^p(\Xb) + \mathcal{V}^p(\Yb)) \label{eq:ass_eq2}\\
	\|\mathcal{D}^k \psi\| &\leq C_k\mathcal{V}^{p_k} \label{eq:ass_eq3}
\end{align}

Compared to the SGLD case \cite{VollmerZT:arxiv15}, in our proofs, we only need $k$ be up to 3 
in \eqref{eq:ass_eq3} instead of 4. More specifically, the proof for the bias only needs $k$ be up 
to 0 given other assumptions in this paper, and the proof for the MSE needs $k$ be up to 3. 

As long as the corresponding SDE is hypoelliptic, meaning that the Brownian motion $W$ is able
to propagate to the other variables of the dynamics \cite{MattinglyST:JNA10}, {\it e.g.}, the model
parameter $\thetab$ in SGHMC, we can
extend Assumption~4.1 of \cite{VollmerZT:arxiv15} to our setting. Thus we have that \eqref{eq:ass_eq1} 
is equivalent to finding a function $\mathcal{V}: \RR^{n} \rightarrow [1, \infty]$ ($n$ is the dimension of
$\xb$, {\it e.g.}, including the momentum in SGHMC), which tends to infinity as 
$\xb \rightarrow \infty$, and is twice differentiable with bounded second derivatives and satisfies the 
following conditions:
\begin{itemize}
	\item [1.] $\mathcal{V}$ is a Lyapunov function of the SDE, {\it i.e.}, there exists constants $\alpha, \beta > 0$,
	such that for $\xb \in \RR^{n}$, we have
	$\langle \nabla_{\xb}\mathcal{V}(\xb), F(\xb) \rangle \leq -\alpha \mathcal{V}(\xb) + \beta$.
	\item [2.] There exists an exponent $p_H \geq 2$ such that 
	$\mathbb{E}\left\|\tilde{F}(\xb) - \mathbb{E}_{s}\tilde{F}(\xb)\right\| \lesssim \mathcal{V}^{p_H}(\xb)$,
	where $\mathbb{E}_{s}$ means expectation with respect to the random permutation of the data,
	$\mathbb{E}$ means expectation with respect to the randomness of the dynamic with Brownian motion.
	Furthermore, for $\xb \in \RR^{n}$, we have: $\left\|\nabla_{\xb}\mathcal{V}(\xb)\right\|^2 + \left\|F(\xb)\right\|^2 \lesssim \mathcal{V}(\xb)$.
\end{itemize}

Similar to \cite{VollmerZT:arxiv15}, \eqref{eq:ass_eq2} is an extra condition that needs to be satisfied,
and \eqref{eq:ass_eq3} is more subtle and needs more assumptions to verify in this case. We will not
address these issues because it is out of the scope of the paper.

\section{The Proof of Theorem~\ref{theo:bias1}}\label{sec:bias_proof}

\begin{proof}

For an SG-MCMC with a $K$th-order integrator, according to Definition~\ref{def:k-order-integrator}
and \eqref{eq:sgmcmc_integrator}, we have: 
	\begin{align} \label{eq:split_flow1}
		\mathbb{E}[\psi(\Xb_{lh})] &= \tilde{P}_h^l \psi(\Xb_{(l-1)h}) = e^{h\tilde{\Lcal}_l} \psi(\Xb_{(l-1)h}) + O(h^{K+1}) \nonumber\\
		&= \left(\mathbb{I} + h\tilde{\Lcal}_l\right) \psi(\Xb_{(l-1)h}) + \sum_{k=2}^K\frac{h^k}{k!}\tilde{\Lcal}_l^k\psi(\Xb_{(l-1)h}) + O(h^{K+1})~,
	\end{align}
	where $\mathbb{I}$ is the identity map. Sum over $l = 1, \cdots, L$ in \eqref{eq:split_flow1}, take expectation on both sides, and use the relation
	$\tilde{\Lcal}_l = \Lcal + \Delta V_l$ to expand the first order term. We obtain
	\begin{align*}
		\sum_{l=1}^{L}\mathbb{E}[\psi(\Xb_{lh})] =& \psi(\Xb_0) + \sum_{l=1}^{L-1} \mathbb{E}[\psi(\Xb_{lh})] + h\sum_{l=1}^{L} \mathbb{E}[\mathcal{L}\psi(\Xb_{(l-1)h})] \\
		&+ h\sum_{l=1}^L \mathbb{E}[\Delta V_l \psi(\Xb_{(l-1)h})] + \sum_{k=2}^K\frac{h^k}{k!}\sum_{l=1}^L \EE[\tilde{\Lcal}_l^k\psi(\Xb_{(l-1)h})] + O(L h^{K+1}).
	\end{align*}
	Divide both sides by $Lh$, use the Poisson equation \eqref{eq:PoissonEq1}, and reorganize terms. We have:
	\begin{align}\label{eq:expansion11}
		&\mathbb{E}[\frac{1}{L}\sum_l\phi(\Xb_{lh}) - \bar{\phi}] = \frac{1}{L}\sum_{l=1}^{L} \mathbb{E}[\mathcal{L}\psi(\Xb_{(l-1)h})] \\
		=&\frac{1}{Lh}\left(\mathbb{E}[\psi(\Xb_{lh})] - \psi(\Xb_0)\right)
		- \frac{1}{L}\sum_l \mathbb{E}[\Delta V_l\psi(\Xb_{(l-1)h})] - \sum_{k=2}^K\frac{h^{k-1}}{k!L}\sum_{l=1}^L \EE[\tilde{\Lcal}_l^k\psi(\Xb_{(l-1)h})] + O(h^K) \nonumber
	\end{align}
	To transform terms containing $\tilde{\Lcal}_l^k (k \geq 2)$ to high-order terms, based on ideas from \cite{MattinglyST:JNA10},
	we apply the following procedure.
	First replace $\psi$ with $\tilde{\Lcal}_l^{K-1}\psi$ from 
	\eqref{eq:split_flow1} to \eqref{eq:expansion11}, and apply the same logic for $\tilde{\Lcal}_l^{K-1}\psi$ 
	as for $\psi$ in the above derivations, but this time expand in \eqref{eq:split_flow1}
	up to the order of $O(h^2)$, instead of the previous order $O(h^{K+1})$. After simplification, we obtain:
	\begin{align}\label{eq:expansion21}
		&\sum_l \mathbb{E}[\tilde{\Lcal}_l^K \psi(\Xb_{(l-1)h})]
		= O\left(\frac{1}{h} + Lh\right)
	\end{align}
	Similarly, replace $\psi$ with $\tilde{\Lcal}_l^{K-2}\psi$ from 
	\eqref{eq:split_flow1} to \eqref{eq:expansion11}, follow the same derivations as for $\tilde{\Lcal}_l^{K-1}\psi$, 
	but expand in \eqref{eq:split_flow1} up to the order of $O(h^3)$ instead of $O(h^2)$. We have:
	\begin{align}\label{eq:expansion22}
		&\sum_l \mathbb{E}[\tilde{\Lcal}_l^{K-1} \psi(\Xb_{(l-1)h})]
		= O\left(\frac{1}{h} + Lh^2\right) + \frac{h}{2}\sum_{l=1}^L \mathbb{E}[\tilde{\Lcal}_l^{K} \psi(\Xb_{(l-1)h})]
		= O\left(\frac{1}{h} + Lh^2\right)~,
	\end{align}
	where the last equation in \eqref{eq:expansion22} is obtained by substituting \eqref{eq:expansion21} into it 
	and collecting low order terms.
	By induction on $k$, it is easy to show that for $2 \leq k \leq K$, we have:
	\begin{align}\label{eq:expansion23}
		\sum_l \mathbb{E}[\tilde{\Lcal}_l^{k} \psi(\Xb_{(l-1)h})]
		= O\left(\frac{1}{h} + Lh^{K-k+1}\right)~,
	\end{align}
	Substituting \eqref{eq:expansion23} into \eqref{eq:expansion11}, after simplification, we have:
	$\mathbb{E}\left(\frac{1}{L}\sum_l\phi(\Xb_{lh}) - \bar{\phi}\right)$
	\begin{align*}
		=\frac{1}{Lh}\underbrace{\left(\mathbb{E}[\psi(\Xb_{lh})] - \psi(\Xb_0)\right)}_{C_1}
		- \frac{1}{L}\sum_l \mathbb{E}[\Delta V_l\psi(\Xb_{(l-1)h})] - \sum_{k=2}^KO\left(\frac{h^{k-1}}{Lh} + h^{K}\right) + C_3 h^K~,
	\end{align*}	
	for some $C_3 \geq 0$.
	According to the assumption, the term $C_1$ is bounded. As a result, collecting low order terms, the bias can be expressed as:
	\begin{align*}
		\left|\mathbb{E}\hat{\phi} - \bar{\phi}\right| &= \left|\mathbb{E}\left(\frac{1}{L}\sum_l\phi(\Xb_{lh}) - \bar{\phi}\right)\right|
		= \left|\frac{C_1}{Lh} - \frac{\sum_l \mathbb{E}\Delta V_l\psi(\Xb_{(l-1)h})}{L} + C_3 h^K\right| \\
		&\leq \left|\frac{C_1}{Lh}\right| + \left| \frac{\sum_l \mathbb{E}\Delta V_l\psi(\Xb_{(l-1)h})}{L}\right| + \left|C_3 h^K\right|
		= O\left(\frac{1}{Lh} + \frac{\sum_l \left\|\mathbb{E}\Delta V_l\right\|}{L} + h^K\right)~,
	\end{align*}
	where the last equation follows from the finiteness assumption of $\psi$, $\|\cdot\|$ denotes the operator norm
	and is bounded in the space of $\psi$ due to the assumptions. 
	This completes the proof.
\end{proof}

\section{The Proof of Theorem~\ref{theo:MSE}}

\begin{proof}
For a $K$-order integrator, from Theorem~\ref{theo:bias1}, we can expand $\mathbb{E}\left(\psi(\Xb_{lh})\right)$ as:
\begin{align*}
	\mathbb{E}\left(\psi(\Xb_{lh})\right) = \left(\mathbb{I} + h(\mathcal{L} + \Delta V_l)\right) \psi(\Xb_{(l-1)h}) + \sum_{k=2}^K\frac{h^k}{k!}\tilde{\mathcal{L}}_l^k \psi(\Xb_{(l-1)h}) + O(h^{K+1})~.
\end{align*}
Sum over $l$ from 1 to $L+1$ and simplify, we have:
\begin{align*}
	\sum_{l=1}^L\mathbb{E}\left(\psi(\Xb_{lh})\right) &= \sum_{l=1}^L \psi(\Xb_{(l-1)h}) + h\sum_{l=1}^L \mathcal{L}\psi(\Xb_{(l-1)h}) + h\sum_{l=1}^L \Delta V_l\psi(\Xb_{(l-1)h}) \\
	&+ \sum_{k=2}^K\frac{h^k}{k!}\sum_{l=1}^L\tilde{\mathcal{L}}_l^k \psi(\Xb_{(l-1)h}) + O(L h^{K+1})~.
\end{align*}
Substitute the Poisson equation \eqref{eq:PoissonEq1} into the above equation, divide both sides by $Lh$ 
and rearrange related terms, we have
\begin{align*}
	\hat{\phi} - \bar{\phi} &= \frac{1}{Lh}\left(\mathbb{E}\psi(\Xb_{Lh}) - \psi(\Xb_0)\right) - \frac{1}{Lh}\sum_{l=1}^{L}\left(\mathbb{E}\psi(\Xb_{(l-1)h}) - \psi(\Xb_{(l-1)h})\right) \\
	&- \frac{1}{L}\sum_{l=1}^L \Delta V_l\psi(\Xb_{(l-1)h})
	- \sum_{k=2}^K\frac{h^{k-1}}{2L}\sum_{l=1}^L\tilde{\mathcal{L}}_l^k \psi(\Xb_{(l-1)h}) + O(h^K)
\end{align*}
Taking square and expectation on both sides, since the terms $\left(\mathbb{E}\psi(\Xb_{(l-1)h}) - \psi(\Xb_{(l-1)h})\right) $ and
$\Delta V_l\psi(\Xb_{(l-1)h})$ are martingale, it is then easy to see there exists some positive constant $C$, such that
\begin{align}\label{eq:mse1}
	\mathbb{E}\left(\hat{\phi} - \bar{\phi}\right)^2 &\leq C\mathbb{E}\left(\underbrace{\frac{\left(\mathbb{E}\psi(\Xb_{Lh}) - \psi(\Xb_0)\right)^2}{L^2h^2}}_{A_1} + \underbrace{\frac{1}{L^2h^2}\sum_{l=1}^L\left(\mathbb{E}\psi(\Xb_{(l-1)h}) - \psi(\Xb_{(l-1)h})\right)^2}_{A_2} \right.\nonumber\\
	&+ \left.\frac{1}{L^2}\sum_{l=1}^L \Delta V_l^2\psi(\Xb_{(l-1)h}) + \underbrace{\sum_{k=2}^K\frac{h^{2(k-1)}}{k!L^2}\left(\sum_{l=1}^L\tilde{\mathcal{L}}_l^k \psi(\Xb_{(l-1)h})\right)^2}_{A_3} + h^{2K}\right) 
\end{align}
$A_1$ is easily bounded by the assumption that $\|\psi\| \leq V^{p_0} < \infty$, the expectation of $A_3$ can also be shown
to be bounded later in \eqref{eq:mse_bound1}. Now we show that 
$A_2$ is bounded as well by deriving the following bound:
$\mathbb{E}\left(\psi(\Xb_{lh})\right) - \psi(\Xb_{lh}) \leq C_1 \sqrt{h} + O(h)$ for $C_1 \geq 0$.
To do this, it is enough to consider the 2nd order symmetric splitting scheme, as higher order integrators generally
introduce higher order errors. Furthermore, we see that different splitting schemes, {\it e.g.}, ABOBA and OABAO, are essentially 
equivalent as long as they are symmetric \cite{LeimkuhlerM:AMRE13}, 
thus we focus on the ABOAB scheme in the proof. Let the flow propagators (mappings) of
`A' , `B' and `O' be denoted as $\tilde{\varphi}_h^A$, $\tilde{\varphi}_h^B$ and $\tilde{\varphi}_h^{O_l}$ respectively. 
Since $\tilde{\varphi}_h^A$ and $\tilde{\varphi}_h^B$ are deterministic, we combine
them and use $\tilde{\varphi}_h^{AB}$ to represent the composition flow $\tilde{\varphi}_h^A \circ \tilde{\varphi}_h^B$. We further
decompose $\tilde{\varphi}_h^{O_l}$ into the deterministic part $\tilde{\varphi}_h^O$ and the stochastic part
$\tilde{\varphi}_h^{\zetab}$ from the brownian motion, then in the iteration for the current minibatch, we can 
express the flow evolution as:
\begin{align}\label{eq:X_l}
	\Xb_{lh} &= \tilde{\varphi}_h^{AB} \circ \left(\tilde{\varphi}_h^O \circ \tilde{\varphi}_h^{\zeta}\right) \circ \tilde{\varphi}_h^{AB} (\Xb_{(l-1)h}) \nonumber\\
	&= \tilde{\varphi}_h^{AB} \left(\tilde{\varphi}_h^O \left(\tilde{\varphi}_h^{AB}(\Xb_{(l-1)h})\right) + \sqrt{2Dh}\zetab_l\right)~,
\end{align}
where $\zetab_{l}$ is a $n$-dimensional independent Gaussian random variables.

From Assumption~\ref{ass:assumption1} we know that both $\tilde{\varphi}_h^{O}$ and $\tilde{\varphi}_h^{AB}$
have bounded derivatives. To simplify the representation, we denote 
$\tilde{\Xb}_{l} \triangleq \tilde{\varphi}_h^{AB}\left(\tilde{\varphi}_h^{O}\left(\tilde{\varphi}_h^{AB}(\Xb_{(l-1)h})\right)\right)$.
Now we can expanded $\Xb_{lh}$ from~\eqref{eq:X_l} using Taylor expansion as:
\begin{align}\label{eq:flowdiff}
	\Xb_{lh} &= \tilde{\varphi}_h^{AB}\left(\tilde{\varphi}_h^{O}\left(\tilde{\varphi}_h^{AB}(\Xb_{(l-1)h})\right) + \sqrt{2Dh} \zeta_l\right) \nonumber\\
	&= \tilde{\Xb}_l + \mathcal{D}\tilde{\Xb}_{l}\left[\sqrt{2Dh} \zetab_l\right] + \frac{1}{2}\mathcal{D}^2\tilde{\Xb}_{l}\left[\sqrt{2Dh} \zetab_l, \sqrt{2Dh} \zetab_l\right] +  O(h\zetab_l^2)
\end{align}
Using the relation~\eqref{eq:flowdiff}, for the solution $\psi$ of the Poisson equation \eqref{eq:PoissonEq1} applied on $\Xb_{lh}$,
we can bow expand it up to 3 orders from the Taylor theory:
\begin{align}\label{eq:expx_t_x}
	&\psi(\Xb_{lh}) = \psi\left(\tilde{\varphi}_h^{AB}\left(\tilde{\varphi}_h^{O}\left(\tilde{\varphi}_h^{AB}(\Xb_{(l-1)h})\right)\right) + \mathcal{D}\tilde{\Xb}_{l}\left[\sqrt{2Dh} \zetab_l\right] + \frac{1}{2}\mathcal{D}^2\tilde{\Xb}_{l}\left[\sqrt{2Dh} \zetab_l, \sqrt{2Dh} \zetab_l\right] +  O(h\zetab_l^2)\right) \nonumber\\
	=& \psi\left(\tilde{\Xb}_l\right) + \underbrace{\mathcal{D}\psi(\tilde{\Xb}_l)\left[\mathcal{D}\tilde{\Xb}_{l}\left[\sqrt{2Dh} \zetab_l\right]\right]}_{M_1} + \underbrace{\frac{1}{2}\mathcal{D}\psi(\tilde{\Xb}_l)\left[\mathcal{D}^2\tilde{\Xb}_{l}\left[\sqrt{2Dh} \zetab_l, \sqrt{2Dh} \zetab_l\right]\right]}_{S_1} \nonumber\\
	&+ \underbrace{\frac{1}{2}\mathcal{D}^2\psi(\tilde{\Xb}_l)\left[\left(\mathcal{D}\tilde{\Xb}_{l}\left[\sqrt{2Dh} \zeta_l\right] + \frac{1}{2}\mathcal{D}^2\tilde{\Xb}_{l}\left[\sqrt{2Dh} \zetab_l, \sqrt{2Dh} \zetab_i\right]\right)^{2\bigotimes}\right]}_{S_2} \\
	&+ \underbrace{\frac{1}{2}\int_0^1 s^2 \mathcal{D}^3\psi(s\Xb_{(l-1)h} + (1 - s)\tilde{\Xb}_{l})\left[\left(\mathcal{D}\tilde{\Xb}_{l}\left[\sqrt{2Dh} \zetab_l\right] + \frac{1}{2}\mathcal{D}^2\tilde{\Xb}_{l}\left[\sqrt{2Dh} \zetab_l, \sqrt{2Dh} \zetab_l\right]\right)^{3\bigotimes}\right]}_{R} \nonumber
\end{align}
where $[(\Xb)^{N\bigotimes}] \triangleq [\underbrace{\Xb, \cdots, \Xb}_{N}]$.

Note that the vector fields inside the brackets in the above expression are all bounded due to 
Assumption~\ref{ass:assumption1}. As a result, we can show that $M_1, S_1, S_2$ and $R$ 
are bounded by the boundedness assumption on $\psi$ and its
derivatives. Specifically, in the following we will use $a \lesssim b$ to represent there is a $C \geq 0$ such that
$a \leq Cb$. 
Let $\tilde{\varphi}_{h_{l}}(\xb) \triangleq \tilde{\varphi}_h^{OA}\left(\tilde{\varphi}_h^{B}\left(\tilde{\varphi}_h^{OA}(\Xb_{lh} + \xb)\right)\right)$, according to the definition of directional derivative, we have
\begin{align*}
	\mathcal{D}\tilde{\Xb}_{l}\left[\sqrt{2Dh} \zetab_l\right] &\triangleq \lim_{\alpha \rightarrow 0} \frac{\tilde{\varphi}_{h_{l-1}}(\alpha \sqrt{2Dh} \zetab_l) - \tilde{\varphi}_{h_{l-1}}(0)}{\alpha} \\
	&= \lim_{\alpha \rightarrow 0} \frac{\alpha \sqrt{2Dh} J(0) \zetab_l + O(\alpha)}{\alpha} = \sqrt{2Dh} J(0) \zetab_i~,
\end{align*}
where $J(\xb)$ is the Jacobian of $\tilde{\varphi}_{h_{l-1}}(\xb)$ and is bounded. Thus
\begin{align}\label{eq:largeorder}
	\mathbb{E}M_1^2 \lesssim h \sup_l \mathbb{E}\mathcal{V}_l^{2p_1} \lesssim h~.
\end{align}
Similarly, for $S_1$ and $S_2$, using the assumptions in the theory, we have
\begin{align*}
	\mathbb{E}S_1^2 &\lesssim h^2 \sup_l \mathbb{E}\mathcal{V}_l^{2p_1} \lesssim h^2 \\
	\mathbb{E}S_2^2 &\lesssim (\sqrt{h} + h)^2 \sup_l \mathbb{E}\mathcal{V}_l^{2p_2} \lesssim (\sqrt{h} + h)^2~.
\end{align*}
For $R$, using Assumption~\ref{ass:assumption1}, we have
\begin{align*}
	\mathbb{E}R^2 &\lesssim (\mathbb{E}\mathcal{V}(\Xb_{(l-1)h})^{2p_{3}} + \mathbb{E}\mathcal{V}(\tilde{\Xb}_l)^{2p_3}) \left\|\mathcal{D}\tilde{\varphi}_{h_{l-1}}^{OA}\left[\sqrt{2Dh} \zetab_i\right] + \frac{1}{2}\mathcal{D}^2\tilde{\varphi}_{h_{l-1}}^{OA}\left[\sqrt{2Dh} \zetab_l, \sqrt{2Dh} \zetab_l\right]\right\|^3 \\
	&\lesssim h^{3}~.
\end{align*}

The expectation of $\psi(\Xb_{lh})$ can be similarly bounded. Collecting low order terms, we have
\begin{align*}
	\mathbb{E}\left(\mathbb{E}\left(\psi(\Xb_{lh})\right) - \psi(\Xb_{lh})\right)^2 = Ch + O(h^{3/2})~,
\end{align*}
for some $C > 0$. As a result, the expectation of the $A_2$ term in \eqref{eq:mse1} can be bounded
using the above derived bound on $\mathbb{E}\left(\psi(\Xb_{lh})\right) - \psi(\Xb_{lh})$.
\begin{align}\label{eq:sum_bound}
	\frac{1}{L^2h^2}\sum_{l}\mathbb{E}\left(\mathbb{E}\psi(\Xb_{lh}) - \psi(\Xb_{lh})\right)^2 = \frac{C}{L h} + O(\frac{1}{L\sqrt{h}})~.
\end{align}
Substitute \eqref{eq:sum_bound} into \eqref{eq:mse1} we can bound the MSE as:
\begin{align}
	&\mathbb{E}\left(\hat{\phi} - \bar{\phi}\right)^2 \nonumber\\
	\lesssim& \frac{\frac{1}{L}\sum_l\mathbb{E}\Delta V_l^2\psi(\Xb_{(l-1)h})}{L} + \sum_{k=2}^K\frac{h^{2(k-1)}}{2L^2}\mathbb{E}\left(\sum_{l=1}^L\tilde{\mathcal{L}}_l^k \psi(\Xb_{(l-1)h})\right)^2 + \frac{1}{Lh} + \frac{1}{L^2h^2} + O(h^{2K}) \nonumber\\
	=& \frac{\frac{1}{L}\sum_l\mathbb{E}\Delta V_l^2\psi(\Xb_{(l-1)h})}{L} + \underbrace{\sum_{k=2}^K\frac{h^{2(k-1)}}{2L^2}\left(\sum_{l=1}^L\mathbb{E}\left[\tilde{\mathcal{L}}_l^k \psi(\Xb_{(l-1)h})\right]\right)^2}_{A_1} + \frac{1}{Lh} + \frac{1}{L^2h^2} \nonumber\\
	&+ \sum_{k=2}^K\underbrace{\frac{h^{2(k-1)}}{2L^2}\mathbb{E}\left(\sum_{l=1}^L\left(\tilde{\mathcal{L}}_l^k \psi(\Xb_{(l-1)h}) - \mathbb{E}\tilde{\mathcal{L}}_l^k \psi(\Xb_{(l-1)h})\right)\right)^2}_{A_2} + O(h^{2K}) \label{eq:mse_bound1}\\
	\leq& C\left(\frac{\frac{1}{L}\sum_l\mathbb{E}\left\|\Delta V_l\right\|^2}{L} + \frac{1}{Lh} + h^{2K}\right)\label{eq:mse_bound2}
\end{align}
for some $C > 0$, where \eqref{eq:mse_bound1} follows by using the fact that $\mathbb{E}[\Xb^2] = \mathbb{E}[(\Xb - \mathbb{E}\Xb)^2] + (\mathbb{E}\Xb)^2$ for a random variable $\Xb$. \eqref{eq:mse_bound2} follows by using the bounds in \eqref{eq:expansion23}
on $A_1$, which is bounded by $O(\frac{1}{L^2h^2} + h^{2K})$. For $A_2$, because the terms 
$\left(\tilde{\mathcal{L}}_l^k \psi(\Xb_{(l-1)h}) - \mathbb{E}\tilde{\mathcal{L}}_l^k \psi(\Xb_{(l-1)h})\right)$ are martingale, we have:
\begin{align*}
	A_2 &\lesssim \frac{h^{2(k-1)}}{2L^2} \sum_{l=1}^L \mathbb{E} \left(\tilde{\mathcal{L}}_l^k \psi(\Xb_{(l-1)h}) - \mathbb{E}\tilde{\mathcal{L}}_l^k \psi(\Xb_{(l-1)h})\right)^2 \\
	&\lesssim \frac{1}{Lh}\left(\frac{h^{2k-1}}{L}\sum_{l=1}^L\mathbb{E}(\tilde{\mathcal{L}}_l^k \psi(\Xb_{(l-1)h}))^2\right) + O\left(\frac{1}{L^2h^2} + h^{2K}\right) = O\left(\frac{1}{Lh} + h^{2K}\right)
\end{align*}
where we have used \eqref{eq:expansion23} and the fact that $\mathbb{E}\tilde{\mathcal{L}}_l^k \psi(\Xb_{(l-1)h})$ is bounded. 
Collecting low order terms we get \eqref{eq:mse_bound2}.
This completes the proof.
\end{proof}

\section{The Proof of Theorem~\ref{theo:invariantmeasure}}

\begin{proof}
Because the splitting scheme is geometric ergodic, for a test function $\phi$, from the ergodic theorem we have
\begin{align}
	\int_{\mathcal{X}} \phi(\xb) \tilde{\rho}_h(\mathrm{d}\xb) = \int_{\mathcal{X}} \mathbb{E}_{\xb} \phi(\Xb_{lh}) \tilde{\rho}_h(\mathrm{d}\xb)
\end{align}
 for $\forall l \geq 0, \forall \xb \in \mathcal{X}$. 
Average over all the samples $\{\Xb_{lh}\}$ and let $l$ approach to $\infty$, we have
\begin{align*}
	\int_{\mathcal{X}} \phi(\xb) \tilde{\rho}_h(\mathrm{d}\xb) = \lim_{L \rightarrow \infty}\int_{\mathcal{X}} \frac{1}{L}\sum_{l=1}^L\mathbb{E}_{\xb}\phi(\Xb_{lh}) \tilde{\rho}_h(\mathrm{d}\xb)~.
\end{align*}
Thus the distance between any invariant measure $\tilde{\rho}_h$ of a high-order integrator and $\rho$ can be bounded as:
\begin{align}
	d(\tilde{\rho}_h, \rho) &= \sup_{\phi}\left|\int_{\mathcal{X}} \phi(\xb) \tilde{\rho}_h(\mathrm{d}\xb) - \int_{\mathcal{X}} \phi(\xb) \rho(\mathrm{d}\xb)\right| \nonumber\\
	&= \sup_{\phi} \lim_{L \rightarrow \infty}\left| \int_{\mathcal{X}}\left[\frac{1}{L}\sum_{l=1}^L\mathbb{E}_{\xb}\phi(\Xb_{lh}) - \bar{\phi}\right] \tilde{\rho}_h(\mathrm{d}\xb) \right| \nonumber\\
	&\leq \sup_{\phi} \lim_{L \rightarrow \infty} \int_{\mathcal{X}}\left|\frac{1}{L}\sum_{l=1}^L\mathbb{E}_{\xb}\phi(\Xb_{lh}) - \bar{\phi}\right| \tilde{\rho}_h(\mathrm{d}\xb) \nonumber\\
	&\leq \sup_{\phi} \lim_{L \rightarrow \infty} \left(\frac{C_1}{Lh} + C_2 h^K \label{eq:bounded_h}\right) \\
	&= Ch^K~, \nonumber
\end{align}
where \eqref{eq:bounded_h} follows by using the result from Theorem~\ref{theo:bias1}.
This completes the proof.
\end{proof}

\section{The Proof of Theorem~\ref{theo:bias_w}}

We separate the proof into proofs for the bias and MSE respectively in the following.

\paragraph{The proof for the bias:}

\begin{proof}
	Following Theorem~\ref{theo:bias1}, in the decreasing step size setting, \eqref{eq:split_flow1} can be written as: 
	\begin{align*} 
		\mathbb{E}&\left(\psi(\Xb_{lh})\right) = \left(\mathbb{I} + h_l\tilde{\mathcal{L}}_l\right) \psi(\Xb_{(l-1)h}) + \sum_{k=2}^K\frac{h_l^k}{k!}\tilde{\mathcal{L}}_l^2\psi(\Xb_{(l-1)h}) + O(h_l^{K+1})~.
	\end{align*}
	Similarly, \eqref{eq:expansion11} can be simplified using the step size sequence $(h_l)$ as:
	\begin{align}\label{eq:expansion1_w}
		&\mathbb{E}\left(\tilde{\phi} - \bar{\phi}\right) \nonumber\\
		=&\frac{1}{S_L}\left(\mathbb{E}\left(\psi(\Xb_{Lh})\right) - \psi(\Xb_0)\right)
		 - \sum_{k=2}^K\sum_{l=1}^L\frac{h_l^k}{k!S_L}\tilde{\mathcal{L}}_l^k\psi(\Xb_{(l-1)h}) + O(\frac{\sum_{l=1}^L h_l^{K+1}}{S_L})
	\end{align}	
	
	Similar to the derivation of \eqref{eq:expansion23}, we can derive the following bounds $k = (2, \cdots, K)$:
	\begin{align}\label{eq:expansion2_w}
		\sum_{l=1}^Lh_l^k\mathbb{E}\tilde{\mathcal{L}}_l^k \psi(\Xb_{(l-1)h}) 
		&= O\left(\sum_{l=1}^L\left((h_l^{k-1} - h_{l-1}^{k-1})\tilde{\mathcal{L}}_l^{k-1}\psi(\Xb_{(l-1)h}) + h_l^{K+1}\right)\right) \nonumber\\
		&= O\left(1 + \sum_{l=1}^Lh_l^{K+1}\right)~.
	\end{align}
	Substitute \eqref{eq:expansion2_w} into \eqref{eq:expansion1_w} and collect low order terms, we have:
	\begin{align}
		\mathbb{E}\left(\tilde{\phi} - \bar{\phi}\right) 
		=\frac{1}{S_L}\left(\mathbb{E}\left(\psi(\Xb_{Lh})\right) - \psi(\Xb_0)\right) + O(\frac{\sum_{l=1}^L h_l^{K+1}}{S_L})~.
	\end{align}
	
	As a result, the bias can be expressed as:
	\begin{align*}
		\left|\mathbb{E}\tilde{\phi} - \bar{\phi}\right|
		\leq& \left|\frac{1}{S_L}\left(\mathbb{E}\left[\psi(\Xb_{Lh})\right] - \psi(\Xb_0)\right)
		  + O(\frac{\sum_{l=1}^L h_l^{K+1}}{S_L})\right| \\
		 \lesssim& \left|\frac{1}{S_L}\right| + \left|\frac{\sum_{l=1}^L h_l^{K+1}}{S_L})\right| \\
		=& O\left(\frac{1}{S_L} + \frac{\sum_{l=1}^L h_l^{K+1}}{S_L}\right)~.
	\end{align*}
	Taking $L \rightarrow \infty$, both terms go to zero by assumption.
	This completes the proof.
\end{proof}

\paragraph{The proof for the MSE:}

\begin{proof}
Following similar derivations as in Theorem~\ref{theo:MSE}, we have that
\begin{align*}
	\sum_{l=1}^L\mathbb{E}\left(\psi(\Xb_{lh})\right) &= \sum_{l=1}^L \psi(\Xb_{(l-1)h}) + \sum_{l=1}^L h_l\mathcal{L}\psi(\Xb_{(l-1)h}) + \sum_{l=1}^L h_l\Delta V_l\psi(\Xb_{(l-1)h}) \\
	&+ \sum_{k=2}^K\sum_{l=1}^L\frac{h_l^k}{k!}\tilde{\mathcal{L}}_l^k \psi(\Xb_{(l-1)h}) + C\sum_{l=1}^L h_l^{K+1}~.
\end{align*}
Substitute the Poisson equation \eqref{eq:PoissonEq1} into the above equation and divided both sides by $S_L$, we have
\begin{align*}
	\hat{\phi} - \bar{\phi} &= \frac{\mathbb{E}\psi(\Xb_{Lh}) - \psi(x_0)}{S_L} + \frac{1}{S_L}\sum_{l=1}^{L-1}\left(\mathbb{E}\psi(\Xb_{(l-1)h}) + \psi(\Xb_{(l-1)h})\right) + \sum_{l=1}^L \frac{h_l}{S_L} \Delta V_l\psi(\Xb_{(l-1)h}) \\
	&+ \sum_{k=2}^K\sum_{l=1}^L\frac{h_l^k}{k!S_L}\tilde{\mathcal{L}}_l^k \psi(\Xb_{(l-1)h}) + C\frac{\sum_{l=1}^L h_l^3}{S_L}~.
\end{align*}
As a result, there exists some positive constant $C$, such that:
\begin{align}\label{eq:mse1_w}
	\mathbb{E}\left(\hat{\phi} - \bar{\phi}\right)^2 &\leq C\mathbb{E}\left(\frac{1}{S_L^2}\underbrace{\left(\psi(\Xb_0) - \mathbb{E}\psi(\Xb_{Lh})\right)^2}_{A_1} + \underbrace{\frac{1}{S_L^2}\sum_{l=1}^L\left(\mathbb{E}\psi(\Xb_{(l-1)h}) - \psi(\Xb_{(l-1)h})\right)^2}_{A_2}\right. \nonumber\\
	&+ \left.\sum_{l=1}^L \frac{h_l^2}{S_L^2}\left\|\Delta V_l\right\|^2 + \underbrace{\sum_{k=2}^K\left(\sum_{l=1}^L\frac{h_l^k}{k!S_L}\tilde{\mathcal{L}}_l^k \psi(\Xb_{(l-1)h})\right)^2}_{A_3} + \left(\frac{\sum_{l=1}^L h_l^3}{S_L}\right)^2\right)
\end{align}
$A_1$ can be bounded by assumptions, and $A_2$ is shown to be bounded by using the fact that 
$\mathbb{E}\psi(\Xb_{(l-1)h}) - \psi(\Xb_{(l-1)h}) = O(\sqrt{h_l})$ from Theorem~\ref{theo:MSE}. Furthermore, 
similar to the proof of Theorem~\ref{theo:MSE}, the expectation
of $A_3$ can also be bounded by using the formula $\mathbb{E}[\Xb^2] = (\mathbb{E}\Xb)^2 + \mathbb{E}[(\Xb - \mathbb{E}\Xb)^2]$ 
and \eqref{eq:expansion2_w}. It turns out that the resulting terms have order higher than those from the other terms, thus can be ignored
in the expression below.
After some simplifications, \eqref{eq:mse1_w} is bounded by:
\begin{align}\label{eq:mse_decreasing}
	\mathbb{E}\left(\hat{\phi} - \bar{\phi}\right)^2 &\lesssim \sum_l \frac{h_l^2}{S_L^2}\mathbb{E}\left\|\Delta V_l\right\|^2 + \frac{1}{S_L} + \frac{1}{S_L^2} + \left(\frac{\sum_{l=1}^L h_l^{K+1}}{S_L}\right)^2 \nonumber\\
	&= C\left(\sum_l \frac{h_l^2}{S_L^2}\mathbb{E}\left\|\Delta V_l\right\|^2 + \frac{1}{S_L} + \frac{(\sum_{l=1}^L h_l^{K+1})^2}{S_L^2} \right)
\end{align}
for some $C > 0$, this completes the first part of the theorem. 
We can see that according to the assumption, the last two terms in \eqref{eq:mse_decreasing} approach to 0 
when $L \rightarrow \infty$.
If we further assume $\frac{\sum_{l=1}^{\infty} h_l^2}{S_L^2} = 0$, then the first term in \eqref{eq:mse_decreasing}
approaches to 0 because:
\begin{align*}
	\sum_l \frac{h_l^2}{S_L^2}\mathbb{E}\left\|\Delta V_l\right\|^2 \leq \left(\sup_l \mathbb{E}\left\|\Delta V_l\right\|^2\right) \frac{\sum_l h_l^2}{S_L^2} \rightarrow 0~.
\end{align*}
As a result, we have $\lim_{L \rightarrow \infty}\mathbb{E}\left(\hat{\phi} - \bar{\phi}\right)^2 = 0$.
\end{proof}

\section{The Proof of Corollary~\ref{coro:stepsize}}

\begin{proof}
We use the following inequalities to bound the term $\sum_{l=1}^L l^{-\alpha}$:
\begin{align*}
	\int_1^L x^{-\alpha} \mathrm{d}x < \sum_{l=1}^L l^{-\alpha} < 1 + \int_1^{L-1} x^{-\alpha} \mathrm{d}x~.
\end{align*}
This is easily seen to be true by noting that $\int_{l}^{l+1} x^{-\alpha} \mathrm{d}x < l^{-\alpha} \times 1 = l^{-\alpha} < \int_{l-1}^{l} x^{-\alpha} \mathrm{d}x$. 
After simplification, we have
\begin{align}\label{eq:L_bound}
	\frac{1 - L^{1-\alpha}}{\alpha - 1} < \sum_{l=1}^L l^{-\alpha} < \frac{\alpha - (L - 1)^{1 - \alpha}}{\alpha - 1}~.
\end{align}
It is then easy to see that the condition for $\sum_{l=1}^{\infty} l^{-\alpha} = \infty$ is $\alpha \leq 1$.
Moreover, we notice that other step size assumptions reduce to compare $\sum_{l=1}^{\infty} l^{-\alpha}$
and $\sum_{l=1}^{\infty} l^{-\alpha_1}$ for $\alpha < \alpha_1$, which using \eqref{eq:L_bound} has the
following bound:
\begin{align*}
	\frac{\alpha - 1}{\alpha_1 - 1} \frac{1 - L^{1-\alpha_1}}{\alpha - (L - 1)^{1 - \alpha}} < \frac{\sum_{l=1}^L l^{-\alpha_1}}{\sum_{l=1}^L l^{-\alpha}} < \frac{\alpha - 1}{\alpha_1 - 1} \frac{\alpha_1 - (L - 1)^{1 - \alpha_1}}{1 - L^{1-\alpha}}~.
\end{align*}
As long as $0 < \alpha < 1$ and $\alpha_1 > \alpha$, the above lower and upper bound would approach to 0, thus
all the assumptions for the step size sequences are satisfied.
\end{proof}

\section{On the Euler Integrator and Symmetric Splitting Integrator}\label{sec:numerical_integrator}

\subsection{Euler integrator}

We first review the Euler scheme used in SGLD and SGHMC. In SGLD 
the update for $\Xb_{lh}$ ($= \thetab_{lh}$) follows:
\begin{align*}
	\thetab_{lh} = \thetab_{(l-1)h}  - \nabla_{\thetab}\tilde{U}_l(\thetab_{(l-1)h} ) h + \sqrt{2h}\zetab_l~,
\end{align*}
where $h$ is the step size, $\zetab_l$ is a vector of {\it i.i.d.} standard normal random variables. 
In SGHMC ($\Xb_{lh} = (\thetab_{lh}, p_{lh})$), it becomes:
\begin{align*}
	\thetab_{lh} = \thetab_{(l-1)h} + \pb_{(l-1)h} h, \hspace{0.5cm}
	\pb_{lh} = (1 - Dh) \pb_{(l-1)h} - \nabla_{\thetab}\tilde{U}_l (\thetab_{(l-1)h} )h + \sqrt{2Dh}\zetab_l~,
\end{align*}

Based on the update equations, it is easily seen that the corresponding Kolmogorov operators $\tilde{P}_{h}^{l}$ are 
\begin{align}\label{eq:sgld_euler}
	\tilde{P}_{h}^{l} = e^{h\mathcal{L}_1}, \mbox{ where }\mathcal{L}_1 \triangleq -\nabla_{\thetab}\tilde{U}_l(\thetab_{(l-1)h}) \cdot \nabla_{\thetab} + 2I : \nabla_{\thetab} \nabla_{\thetab}^T
\end{align}
for SGLD, and
\begin{align}\label{eq:sghmc_euler}
	\tilde{P}_{h}^{l} = e^{h\mathcal{L}_2} \circ e^{h\mathcal{L}_3}~,
\end{align}
for SGHMC, where 
$\mathcal{L}_2 \triangleq \pb \cdot \nabla_{\thetab}$ and $\mathcal{L}_3 \triangleq -D \pb_{(l-1)h} \cdot \nabla_{\pb} -\nabla_{\thetab}\tilde{U}_l(\thetab) \cdot \nabla_{\pb} + 2DI : \nabla_{\pb} \nabla_{\pb}^T$.

We show in the following Lemma that the Euler integrator is a 1st-order local integrator.

\begin{lemma}
The Euler integrator is a 1st-order local integrator, {\it i.e.},
\begin{align}
	\tilde{P}_{h}^{l} = e^{h\tilde{\mathcal{L}}_l} + O(h^2)~.
\end{align}
\end{lemma}

\begin{proof}
For the SGLD, according to the Kolmogorov's backward equation \eqref{eq:koleq}, for the SGLD, we have
\begin{align}\label{eq:euler_kov}
	\mathbb{E}[f(\thetab_{(l-1)h + t})] = e^{t\tilde{\mathcal{L}}_l} f(\thetab_{(l-1)h}), \hspace{0.5cm} 0 \leq t \leq h~,
\end{align}
where $\tilde{\mathcal{L}}_1 \triangleq -\nabla_{\thetab}\tilde{U}_l(\thetab) \cdot \nabla_{\thetab} + 2I : \nabla_{\thetab} \nabla_{\thetab}^T$.
Note $\tilde{U}_l(\thetab)$ can be expanded by Taylor's expansion to the 1st-order such that (based on $\thetab_{lh} = \thetab_{(l-1)h} + O(h)$):
\begin{align*}
	\tilde{\mathcal{L}}_1 &= -\nabla_{\thetab}\tilde{U}_l(\thetab_{(l-1)h}) \cdot \nabla_{\thetab} + 2I : \nabla_{\thetab} \nabla_{\thetab}^T + O(h) \\
	&= \mathcal{L}_1 + O(h)~.
\end{align*}
Substituting the above into \eqref{eq:euler_kov} and use the definition \eqref{eq:sgld_euler}, we have
\begin{align*}
	\tilde{P}_{h}^{l} = e^{h\tilde{\mathcal{L}}_l} + O(h^2)~.
\end{align*}

For the SGHMC, following similar derivations, we have:
\begin{align*}
	\mathcal{L}_2 = \tilde{\mathcal{L}}_2 &\rightarrow e^{h\mathcal{L}_2} = e^{h\tilde{\mathcal{L}}_2} + O(h^2)~, \\
	e^{h\mathcal{L}_3} &= e^{h\tilde{\mathcal{L}}_3} + O(h^2)~,
\end{align*}
where $\tilde{\mathcal{L}}_2 \triangleq \pb \cdot \nabla_{\thetab}$ and $\tilde{\mathcal{L}}_3 \triangleq -D \pb h \cdot \nabla_{\pb} -\nabla_{\thetab}\tilde{U}_l(\thetab) \cdot \nabla_{\pb} + 2DI : \nabla_{\pb} \nabla_{\pb}^T$ are the splitting for the 
true generator $\tilde{\mathcal{L}}_l$.

Now using the Baker--Campbell--Hausdorff (BCH) formula, we have
\begin{align*}
	e^{h\mathcal{L}_2} \circ e^{h\mathcal{L}_3} &= e^{h\tilde{\mathcal{L}}_2} \circ \left(e^{h\tilde{\mathcal{L}}_3} + O(h^2)\right) \\
	&= e^{h(\tilde{\mathcal{L}}_2 + \tilde{\mathcal{L}}_3) + O(h^2)} + O(h^2)
	= e^{h\tilde{\mathcal{L}}_l} + O(h^2)
\end{align*}
As a result, $\tilde{P}_{h}^{l} = e^{h\tilde{\mathcal{L}}_l} + O(h^2)$ for SGHMC.
\end{proof}

\subsection{Symmetric splitting integrator}\label{sec:SS_integrator}

In symmetric splitting scheme, the generator $\tilde{\Lcal}_l$ is split  into a couple of 
sub-generators which can be solved analytically. For example, in SGHMC, it is split into: 
$\tilde{\Lcal}_l = \mathcal{L}_A  + \mathcal{L}_B + \mathcal{L}_{O_l}$, where
\begin{align*}
	\mathcal{A} \triangleq \mathcal{L}_A = \pb \cdot \nabla_{\thetab}, \;\;
	\mathcal{B} \triangleq \mathcal{L}_B = -D \pb \cdot \nabla_{\pb}, \;\;
	\mathcal{O}_l \triangleq \mathcal{L}_{O_l} = -\nabla_{\thetab}\tilde{U}_l(\thetab) \cdot \nabla_{\pb} + 2D : \nabla_{\pb} \nabla_{{\pb}}^T~.
\end{align*}
These sub-generators correspond to the following analytically solvable SDEs:
\begin{align*}
	A: \left\{\begin{array}{ll}
	\mathrm{d}\thetab &= \pb \mathrm{d}t \\
	\mathrm{d}\pb &= 0
	\end{array}\right.,
	B: \left\{\begin{array}{ll}
	\mathrm{d}\thetab &= 0 \\
	\mathrm{d}\pb &= - D \pb \mathrm{d}t
	\end{array}\right.,
	O: \left\{\begin{array}{ll}
	\mathrm{d}\thetab &= 0 \\
	\mathrm{d}\pb &= -\nabla_\thetab \tilde{U}_l(\thetab) \mathrm{d}t + \sqrt{2D}\mathrm{d}W
	\end{array}\right.
\end{align*}

Based on the splitting, the Kolmogorov operator $\tilde{P}_{h}^{l}$ can be seen to be:
\begin{align*}
\tilde{P}_{h}^{l} \triangleq e^{\frac{h}{2}\mathcal{L}_A} \circ e^{\frac{h}{2}\mathcal{L}_B} \circ e^{h\mathcal{L}_{O_l}} \circ e^{\frac{h}{2}\mathcal{L}_B} \circ e^{\frac{h}{2}\mathcal{L}_A}, 
\end{align*}

We show that the corresponding integrator is a 2nd-order local integrator below.

\begin{lemma}
The symmetric splitting integrator is a 2nd-order local integrator, {\it i.e.},
\begin{align}
	\tilde{P}_{h}^{l} = e^{h\tilde{\mathcal{L}}_l} + O(h^3)~.
\end{align}
\end{lemma}

\begin{proof}

This follows from direct calculation using the BCH formula. Specifically,
\begin{align}
	e^{\frac{h}{2}\mathcal{A}} e^{\frac{h}{2}\mathcal{B}} &= e^{\frac{h}{2}\mathcal{A} + \frac{h}{2}\mathcal{B} + \frac{h^2}{8}[\mathcal{A}, \mathcal{B}] + \frac{1}{96}\left([\mathcal{A}, [\mathcal{A}, \mathcal{B}]] + [\mathcal{B}, [\mathcal{B}, \mathcal{A}]]\right) + \cdots} \label{eq:bch1}\\
	&= e^{\frac{h}{2}\mathcal{A} + \frac{h}{2}\mathcal{B} + \frac{h^2}{8}[\mathcal{A}, \mathcal{B}]} + O(h^3)~, \label{eq:bch2}
\end{align}
where $[X, Y] \triangleq XY - YX$ is the commutator of $X$ and $Y$, \eqref{eq:bch1} follows from the BCH formula, and \eqref{eq:bch2}
follows from Assumption~\ref{ass:assumption1} such that the remainder high order terms are bounded \cite{Kopec:JNA14},
so the error term $O(h^3)$ can be taken out from the exponential map using Taylor expansion.
Similarly, for the other composition, we have
\begin{align*}
	e^{h \mathcal{O}_l} e^{\frac{h}{2}\mathcal{A}} e^{\frac{h}{2}\mathcal{B}} &= e^{h \mathcal{O}_l} \left(e^{\frac{h}{2}\mathcal{A} + \frac{h}{2}\mathcal{B} + \frac{h^2}{8}[\mathcal{A}, \mathcal{B}]} + O(h^3)\right) \\
	&= e^{h \mathcal{O}_l + \frac{h}{2}\mathcal{A} + \frac{h}{2}\mathcal{B} + \frac{h^2}{8}[\mathcal{A}, \mathcal{B}] + \frac{1}{2}[h \mathcal{O}_l, \frac{h}{2}\mathcal{A} + \frac{h}{2}\mathcal{B} + \frac{h^2}{8}[\mathcal{A}, \mathcal{B}]]} + O(h^3) \\
	&= e^{h \mathcal{O}_l + \frac{h}{2}\mathcal{A} + \frac{h}{2}\mathcal{B} + \frac{h^2}{8}[\mathcal{A}, \mathcal{B}] + \frac{h^2}{4}[\mathcal{O}_l, \mathcal{A}] + \frac{h^2}{4}[\mathcal{O}_l, \mathcal{B}]} + O(h^3) \\
	e^{\frac{h}{2}\mathcal{A}} e^{h \mathcal{O}_l} e^{\frac{h}{2}\mathcal{A}} e^{\frac{h}{2}\mathcal{B}} &= e^{\frac{h}{2} \mathcal{A}} \left(e^{h \mathcal{O}_l + \frac{h}{2}\mathcal{A} + \frac{h}{2}\mathcal{B} + \frac{h^2}{8}[\mathcal{A}, \mathcal{B}] + \frac{h^2}{4}[\mathcal{O}_l, \mathcal{A}] + \frac{h^2}{4}[\mathcal{O}_l, \mathcal{B}]} + O(h^3)\right) \\
	&= e^{h \mathcal{O}_l + h\mathcal{A} + \frac{h}{2}\mathcal{B} + \frac{h^2}{4}[\mathcal{A}, \mathcal{B}] + \frac{h^2}{2}[\mathcal{O}_l, \mathcal{B}]} + O(h^3) \\
	\tilde{P}_{h}^{l} \triangleq e^{\frac{h}{2}\mathcal{B}} e^{\frac{h}{2}\mathcal{A}} e^{h \mathcal{Z}} e^{\frac{h}{2}\mathcal{A}} e^{\frac{h}{2}\mathcal{B}}
	&= e^{\frac{h}{2} \mathcal{B}} \left(e^{h \mathcal{O}_l + h\mathcal{A} + \frac{h}{2}\mathcal{B} + \frac{h^2}{4}[\mathcal{A}, \mathcal{B}] + \frac{h^2}{2}[\mathcal{O}_l, \mathcal{B}]} + O(h^3)\right) \\
	&= e^{h \mathcal{O}_l + h\mathcal{A} + h\mathcal{B} + \frac{h^2}{4}[\mathcal{A}, \mathcal{B}] + \frac{h^2}{2}[\mathcal{O}_l, \mathcal{B}] + \frac{h^2}{4}[\mathcal{B}, \mathcal{A}] + \frac{h^2}{4}[\mathcal{B}, \mathcal{O}_l] + \frac{h^2}{8}[\mathcal{B}, \mathcal{B}]} + O(h^3)\\
		&= e^{h(\mathcal{B} + \mathcal{A} + \mathcal{O}_l)} + O(h^3) \\
		&= e^{h(\mathcal{L} + \Delta V_l)} + O(h^3) = e^{h\tilde{\mathcal{L}}_l} + O(h^3)~.
\end{align*}
This completes the proof.
\end{proof}

\section{Mean Flow Error Analysis}\label{sec:meanflow}

In addition to the finite time ergodic error studied previously, we study the mean flow error in this section.
To this end, we first define the exact mean flow to be the solution operator of the Kolmogorov's backward 
equation $\mathbb{E}[f(\Xb_T)] = e^{T\mathcal{L}}f(\Xb_0)$ over time $T = Lh$, {\it i.e.}, $\varphi_T  \triangleq e^{T\mathcal{L}}$.
With our splitting method with stochastic gradients for each minibatch, the mean flow operator consists of a composition
of $L$ local mean flows, {\it i.e.},
$\tilde{\varphi}_T^h \triangleq e^{h\mathcal{L}_{L}} \circ \cdots \circ e^{h\mathcal{L}_{1}}
\triangleq \circ_{l=1}^L e^{h\mathcal{L}_{l}}$, each coming from a minibatch. Our goal in this section is to
compare $\varphi_T$ with $\tilde{\varphi}_T^h$. When the underlying equations of motion are PDEs, {\it i.e.}, no
Brownian motion like the Hamiltonian PDE, $\varphi_T (\Xb_0)$ corresponds to the exact solution trajectory of the 
PDE, whereas $\tilde{\varphi}_T^h$ is the trajectory of splitting methods with stochastic gradients. 
\cite{Betancourt:arxiv15} shows that in this case $\varphi_T (\Xb_0)$ is not close to $\tilde{\varphi}_T^h$ in general.
In the section we extend this result by showing that the conclusion also
holds in the SDE case. We comment that this result is not as surprising as pointed out in \cite{Betancourt:arxiv15} 
because as pointed out in the introduction, such sample wise convergence is not interesting in most real applications.

\begin{theorem}\label{theo:flowerror}
	In SGHMC with the symmetric splitting integrator, the difference between the stochastic mean flow operator $\tilde{\varphi}_T^h$
	and the exact flow operator $\varphi_T $ depends on the running time $T$ and stochastic gradients in each minibatch,
	and is given by the following formula,
	\begin{align*}
		\left\|\tilde{\varphi}_T^h - \varphi_T \right\| = C\left\|\frac{1}{L}\sum_{l=1}^L\Delta V_l + h \left([\mathcal{L}, \Delta V_1] + [\mathcal{L}, \Delta V_L]\right)\right\|T + O(h^2)~,
	\end{align*}
	for some positive constant $C$.
\end{theorem}

We can see from Theorem~\ref{theo:flowerror} that $\tilde{\varphi}_T^h$ is not close to $ \varphi_T$ because of the
uncontrollable terms $\Delta V_l$ with stochastic gradients, thus SG-MCMCs are not sample-wise convergence.

\begin{proof}

First, applying Kolmogorov's backward equation on the original SDE \eqref{eq:itodiffusion} 
with generator $\mathcal{L}$, the true mean flow $\varphi_T(\Xb_0)$ can be expressed as:
\begin{align}
	\varphi_T(\Xb_0) = e^{T\mathcal{L}}(\Xb_0)~.
\end{align}

Now we want to compute the mean flow of the splitting scheme: $\circ_{l=1}^L \hat{\varphi}_{lh}(\Xb_0)$.
We will split the SDE into several parts, with the Brownian motion term going with the stochastic gradient term.
To shown the proof on a different SG-MCMC algorithm, we use the SGHMC with Riemannian information geometry (SGRHMC) 
defined below. Other stochastic gradient MCMC follows similarly. For the SGRHMC, we have

\begin{align}\label{eq:BAOAB}
	\mathrm{d}\left[ \begin{array}{c}
		\thetab \\
		\pb \end{array} \right]
		= &\underbrace{\left[ \begin{array}{c}
		0 \\
		-\left(\nabla_{\thetab}U(\thetab) + \frac{1}{2}\nabla_{\thetab} \log \det G(\thetab)\right)\mathrm{d}t + \sqrt{2D}\mathrm{d}W \end{array} \right]}_{B} \nonumber\\
		&+\underbrace{\left[ \begin{array}{c}
		0 \\
		-DG(\thetab)^{-1}\pb \end{array} \right] \mathrm{d}t}_{A}
		+ \underbrace{\left[ \begin{array}{c}
		G(\thetab)^{-1} \pb \\
		\nu(\thetab, \pb) \end{array} \right] \mathrm{d}t}_{O}
\end{align}

The splitting scheme we consider is the BAOAB scheme. Denote
\begin{align*}
	\mathcal{B} &= \mathcal{L}_B = -\left(\nabla_{\thetab}U(\thetab) + \frac{1}{2}\nabla_{\thetab} \log \det G(\thetab)\right) \cdot \nabla_{\pb} + 2D \bigtriangleup_{\pb} \\
	\mathcal{A} &= \mathcal{L}_{A} = -DG^{-1}\pb \cdot \nabla_{\pb} \\
	\mathcal{O} &= \mathcal{L}_O = G^{-1}\pb \cdot \nabla_{\thetab} + \nu \cdot \nabla_{\pb}~.
\end{align*}

Note that $\mathcal{L} = \mathcal{A} + \mathcal{B} + \mathcal{O}$. In the stochastic gradient case, 
we are using the stochastic gradient from the $l$-th minibatch in the splitting scheme, thus we need to 
modify the operator $\mathcal{B}$ as:
\begin{align*}
	\mathcal{B}_l \triangleq \mathcal{L}_{\mathcal{B}_l} = -\left(\nabla_{\thetab}\tilde{U}_l(\thetab) + \frac{1}{2}\nabla_{\thetab} \log \det G(\thetab)\right) \cdot \nabla_{\pb} + 2D \bigtriangleup_{\pb}~,
\end{align*}
where $\nabla_{\thetab}\tilde{U}_l$ is evaluated on a subset of data. We emphasis the notation that 
$\Delta V_l \triangleq \mathcal{B}_l -\mathcal{B} = \left(\nabla_{\thetab}\tilde{U}_l - \nabla_{\thetab} U\right) \cdot \nabla_{\pb}$,
it can be shown that $\Delta V_l$ commutes with each other, {\it e.g.}, $\Delta V_i \Delta V_j = \Delta V_j \Delta V_i$.

We know from Section~\ref{sec:SS_integrator} that using the symmetric splitting integrator,
the mean flow $\tilde{\varphi}_{h}^{l}$ is close to $e^{h(\mathcal{L} + \Delta V_l)}$ with a $O(h^3)$ error, {\it i.e.},
\begin{align*}
	\varphi_{h}^{l} = e^{h(\mathcal{L} + \Delta V_l)} + O(h^3)~.
\end{align*}

Similar to the proof of the symmetric splitting error, we can calculate the composition of the mean flows
for two mini-batches $i$ and $j$ using the BCH formula as:
\begin{align*}
	\tilde{\varphi}_{h}^{j} \circ \tilde{\varphi}_{h}^{i} &= e^{h(\mathcal{B} + \mathcal{A} + \mathcal{O}) + h \Delta V_i} \circ e^{h(\mathcal{B} + \mathcal{A} + \mathcal{Z}) + h \Delta V_j} + O(h^3) \\
	&= e^{2h \mathcal{L} + h(\Delta V_i + \Delta V_j) + \frac{h^2}{2}[\mathcal{L} + \Delta V_j, \mathcal{L} + \Delta V_i]} + O(h^3) \\
	&= e^{2h \mathcal{L} + h(\Delta V_i + \Delta V_j) + \frac{h^2}{2}\left([\mathcal{L}, \Delta V_i] + [\Delta V_j, \mathcal{L}]\right)} + O(h^3)~,
\end{align*}
where we have used the fact that $\{\Delta V_i\}$ commutes with each other to cancel out the $[\Delta V_i, \Delta V_j]$ term
in the BCH formula. Similarly, for the first three mini-batches $i, j, k$, we have
\begin{align*}
	\tilde{\varphi}_{h}^{k} \circ \tilde{\varphi}_{h}^{j} \circ \tilde{\varphi}_{h}^{i} &= e^{h(\mathcal{B} + \mathcal{A} + \mathcal{O}) + h \Delta V_i} \circ e^{h(\mathcal{B} + \mathcal{A} + \mathcal{O}) + h \Delta V_j} + O(h^3) \\
	&= e^{3h \mathcal{L} + h(\Delta V_i + \Delta V_j + \Delta V_k) + h^2 \left([\mathcal{L}, \Delta V_i] + [\mathcal{L}, \Delta V_k]\right)} + O(h^3)~.
\end{align*}
Similarly, we can do the composition for the entire trajectory, resulting in after simplification:
\begin{align}\label{eq:meanflow}
	\circ_{l=1}^L \tilde{\varphi}_{h}^{l} &= e^{(Lh) \mathcal{L} + (Lh)\frac{1}{L}\sum_{l=1}^L\Delta V_l + (Lh)h \left([\mathcal{L}, \Delta V_1] + [\mathcal{L}, \Delta V_L]\right)} + (Lh)O(h^2) \nonumber\\
	&= e^{T \mathcal{L} + T\frac{1}{L}\sum_{l=1}^L\Delta V_l + T h \left([\mathcal{L}, \Delta V_1] + [\mathcal{L}, \Delta V_L]\right)} + O(h^2)
\end{align}

This completes the first part of the theorem. From Assumption~\ref{ass:assumption1},
we can expand and bound \eqref{eq:meanflow} with the step size $h$ for finite time $T$ as:
\begin{align*}
	\tilde{\varphi}_{T}(\Xb_0) = \left(T \mathcal{L} + T\frac{1}{L}\sum_{l=1}^L\Delta V_l + T h \left([\mathcal{L}, \Delta V_1] + [\mathcal{L}, \Delta V_L]\right)\right)(\Xb_0) + O(h^2)~.
\end{align*}
Similarly, for the true mean flow $\varphi_{T}(\Xb_0)$, it is easy to get
\begin{align*}
	\varphi_{T}(\Xb_0) = \underbrace{e^{h\mathcal{L}} \circ e^{h\mathcal{L}} \circ \cdots \circ e^{h\mathcal{L}}}_{L}
	= T \mathcal{L}(\Xb_0) + O(h^2)~.
\end{align*}

As a result:
\begin{align*}
	\|\varphi_{T}(\Xb_0) - \tilde{\varphi}_{T}(\Xb_0)\| &= \left\|\left(T\frac{1}{L}\sum_{l=1}^L\Delta V_l + T h \left([\mathcal{L}, \Delta V_1] + [\mathcal{L}, \Delta V_L]\right)\right)(\Xb_0) + O(h^2)\right\| \\
	&= \left\|\left(\sum_{l=1}^L\Delta V_l + T\left([\mathcal{L}, \Delta V_1] + [\mathcal{L}, \Delta V_L]\right)\right)(\Xb_0)\right\|h + O(h^2) \\
	&= C\left\|\frac{1}{L}\sum_{l=1}^L\Delta V_l + h \left([\mathcal{L}, \Delta V_1] + [\mathcal{L}, \Delta V_L]\right)\right\|T + O(h^2)
\end{align*}
This completes the proof.
\end{proof}

\section{Additional Experiments}\label{sec:extra_exp}

\subsection{Synthetic data}

We plot the traces of bias and MSE with step size $h \propto L^{\alpha}$ for different rates $\alpha$
in Figure~\ref{fig:bias_mse_gau_rates1}. We can see that when the rates are smaller than the theoretically 
optimal bias rates $\alpha = -1/3$ and MSE rate $\alpha = -1/5$, the bias and MSE tend to decrease faster
than the optimal rates at the beginning, but eventually they slow down and are surpassed by the optimal rates. 
This on the other hand suggests if only a small number of iterations were available in the SG-MCMCs, setting a larger
step size than the theoretically optimal one might be beneficial in practice.

\begin{figure}[t!]
\centering
\subfloat[Bias]{
  \includegraphics[width=0.495\linewidth]{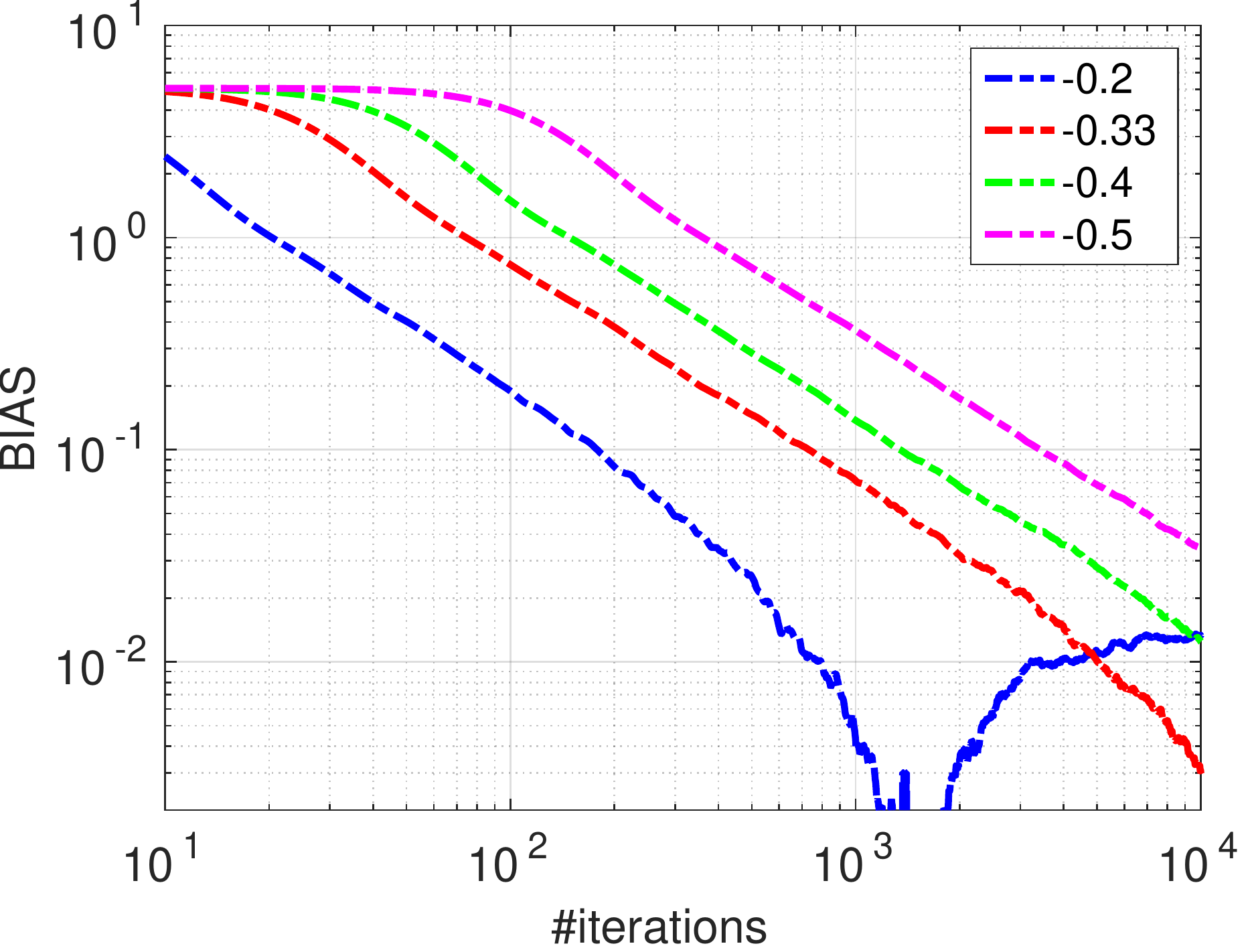}
}
\subfloat[MSE]{
  \includegraphics[width=0.495\linewidth]{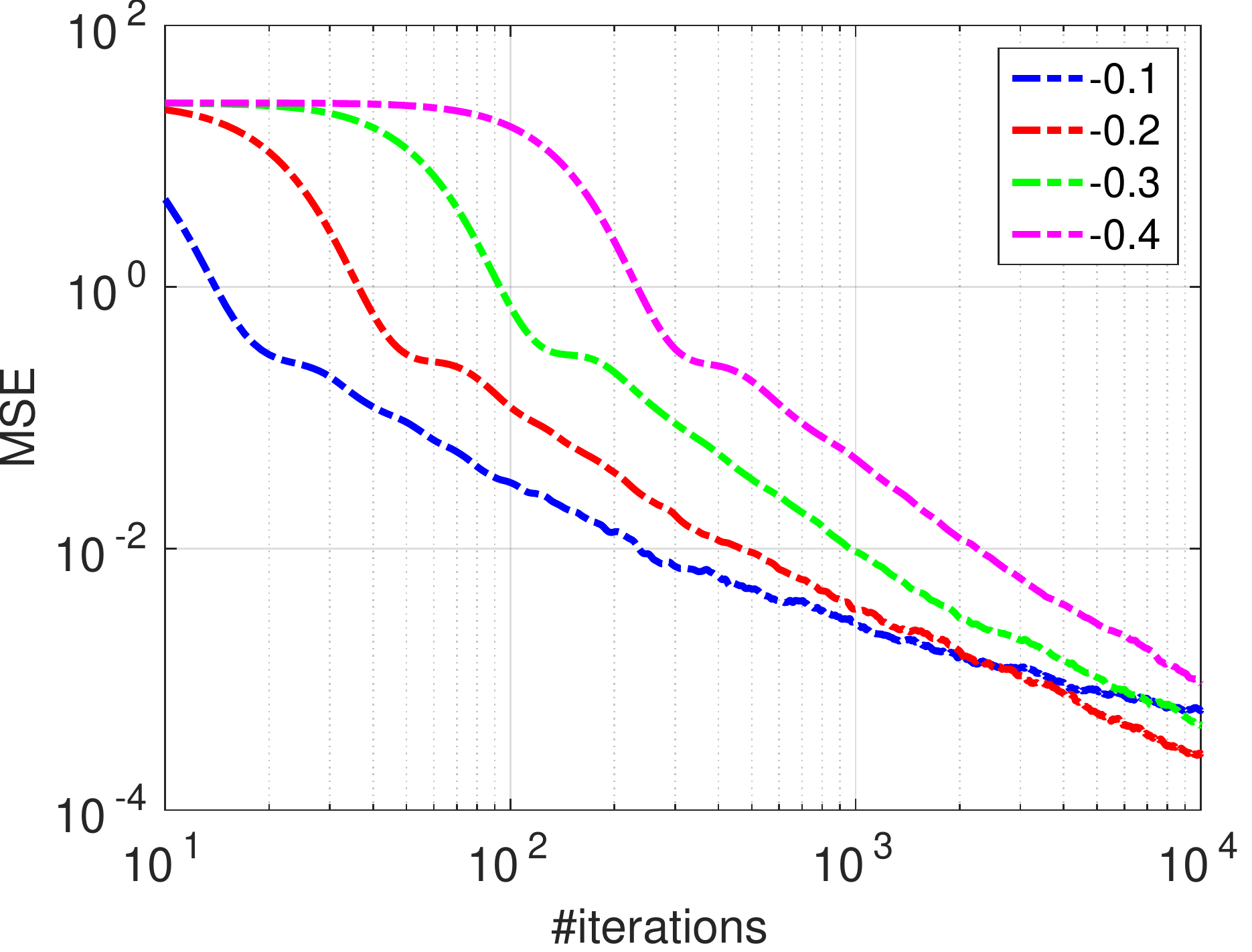}
}
\caption{{\em Bias} and {\em MSE} for SGHMC with different step size rates.}
\label{fig:bias_mse_gau_rates1}
\end{figure}

In addition, Figure~\ref{fig:bias_mse_gau1} shows a comparison of the bias and MSE for SGHMC and SGLD. 
The step sizes are set to $h = C L^{-\alpha}$, with $\alpha$ choosing according to the theory for SGLD and 
SGHMC respectively. To be fair, the constants $C$ are selected via a grid search from 1e-3 to 0.5 with
an interval of 2e-3 for $L = 200$, it is then fixed during other $L$ values. The parameter $D$ in SGHMC 
is selected within $(10, 20, 30)$ as well. As indicated by both our theorems and experiments, SGHMC 
endows a much faster convergence speed than SGHMC on both the bias and MSE.

\begin{figure}[t!]
\centering
\subfloat[Bias]{
  \includegraphics[width=0.495\linewidth]{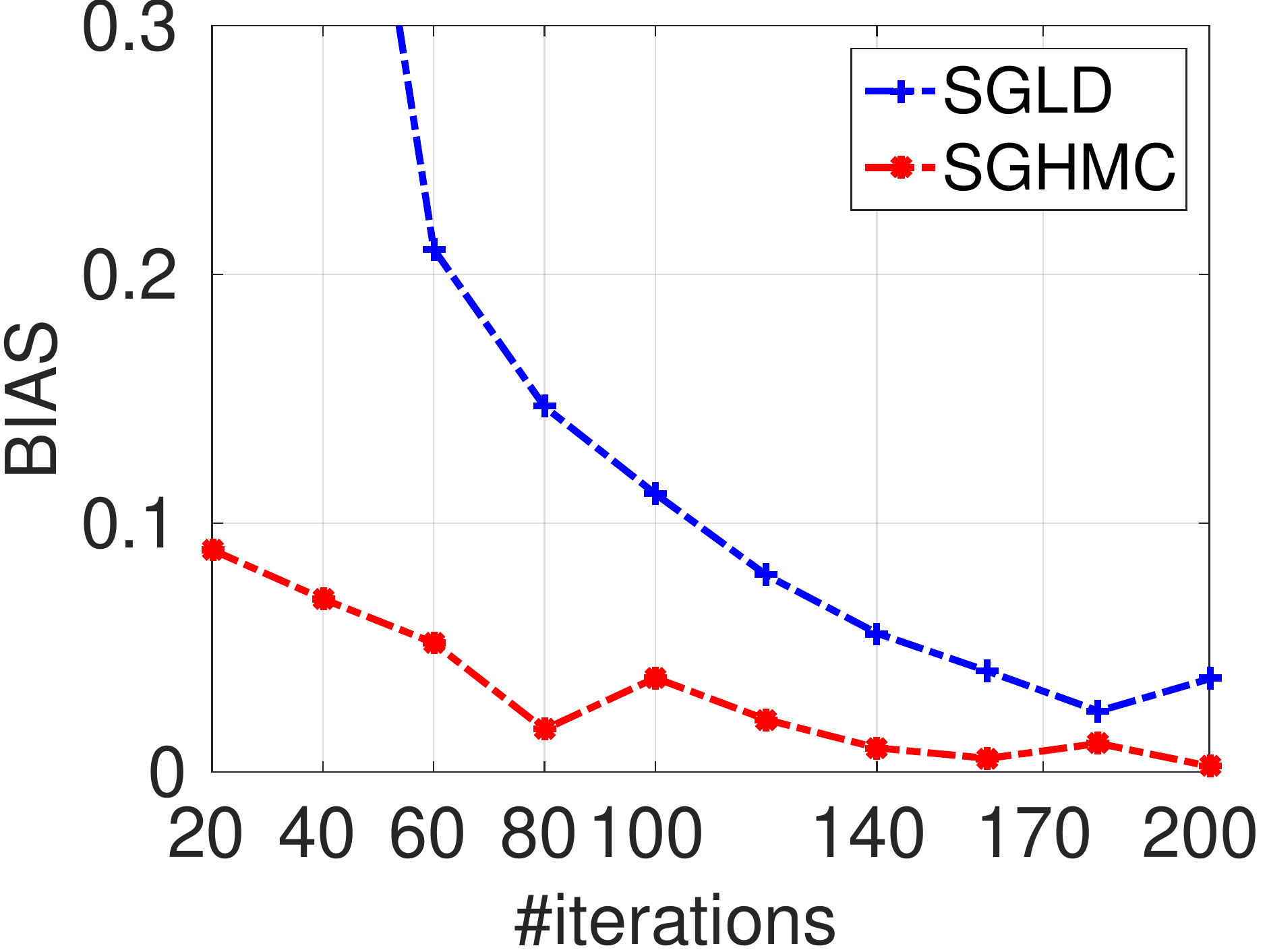}
}
\subfloat[MSE]{
  \includegraphics[width=0.495\linewidth]{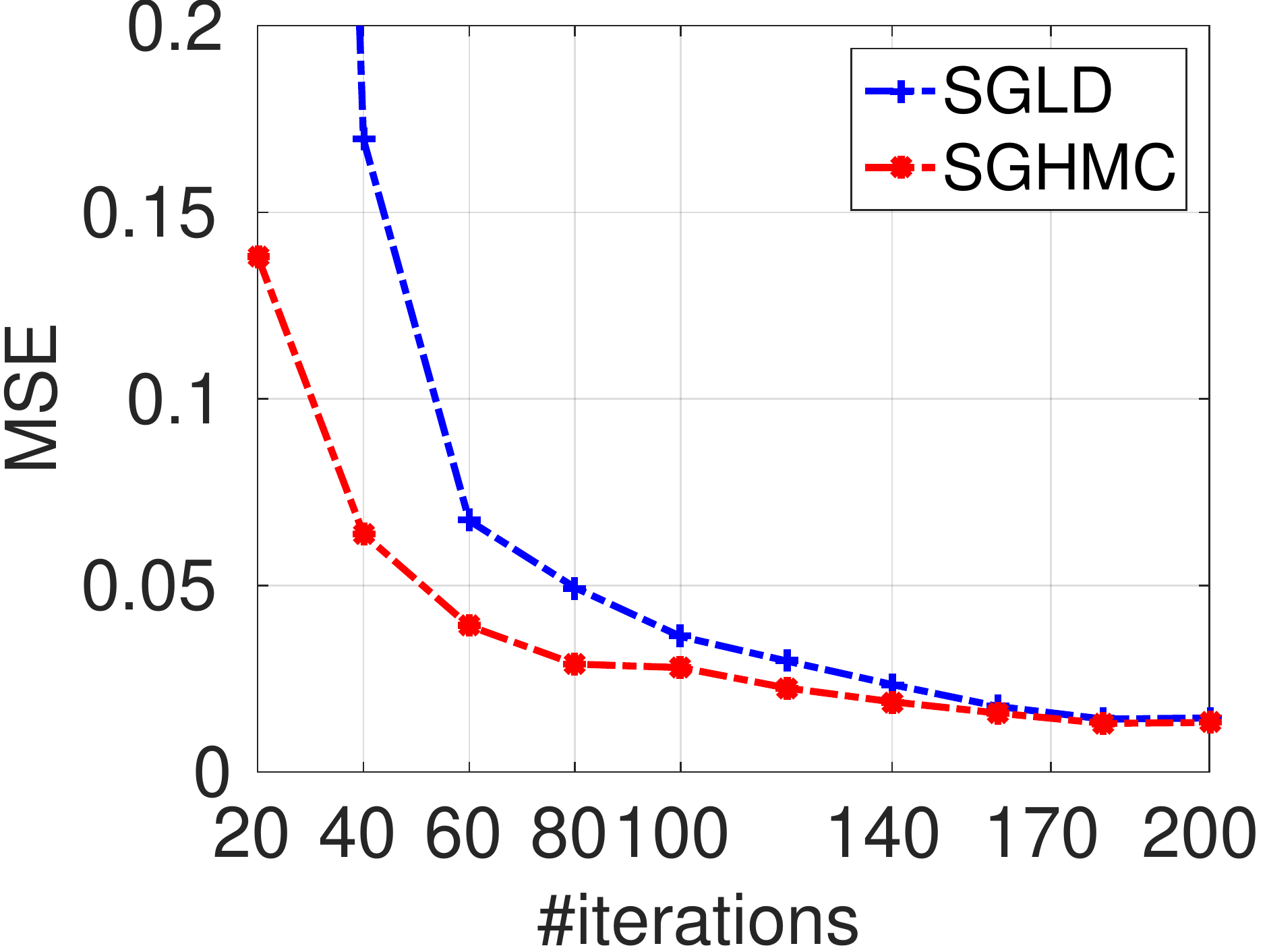}
}
\caption{Comparisons of {\em bias} and {\em MSE} for SGHMC and SGLD on a simple Gaussian model.}
\label{fig:bias_mse_gau1}
\end{figure}

Figure~\ref{fig:bias_mse_gau_decrease_rates1} plots the traces of bias and MSE with decreasing step sizes 
$h \propto l^{\alpha}$ for different rates $\alpha$ in the same Gaussian model. Again we can see that the
optimal decreasing rates agree with the theory.
Figure~\ref{fig:bias_mse_dec_gau1} shows a comparison of bias and MSE for SGHMC and SGLD with
decreasing step sizes $h \propto l^{-\alpha}$ on the same Gaussian model. We follow the same procedure as in 
Section~\ref{sec:finitetimeerr} to select parameters for SGLD and SGHMC. Specifically, the decreasing rate
parameter $\alpha$ is set to $1/2$ and $1/3$ in SGLD and SGHMC for the bias, $1/3$ and $1/5$ for the MSE. 
We can see that SGHMC still obtain a faster convergence speed, though the benefit is not as large as 
using fix step size.

\begin{figure}[t!]
\centering
\subfloat[Bias]{
  \includegraphics[width=0.495\linewidth]{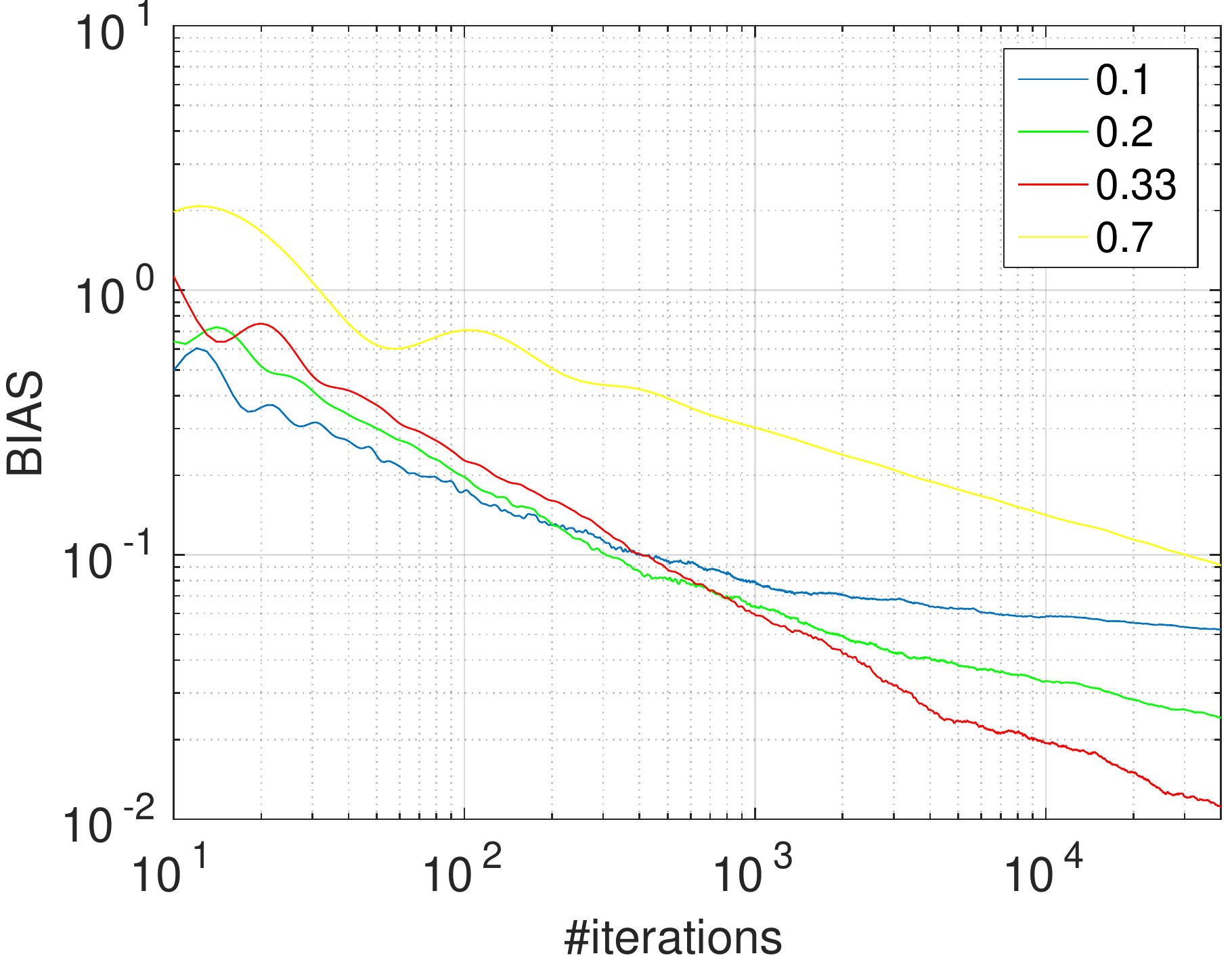}
}
\subfloat[MSE]{
  \includegraphics[width=0.495\linewidth]{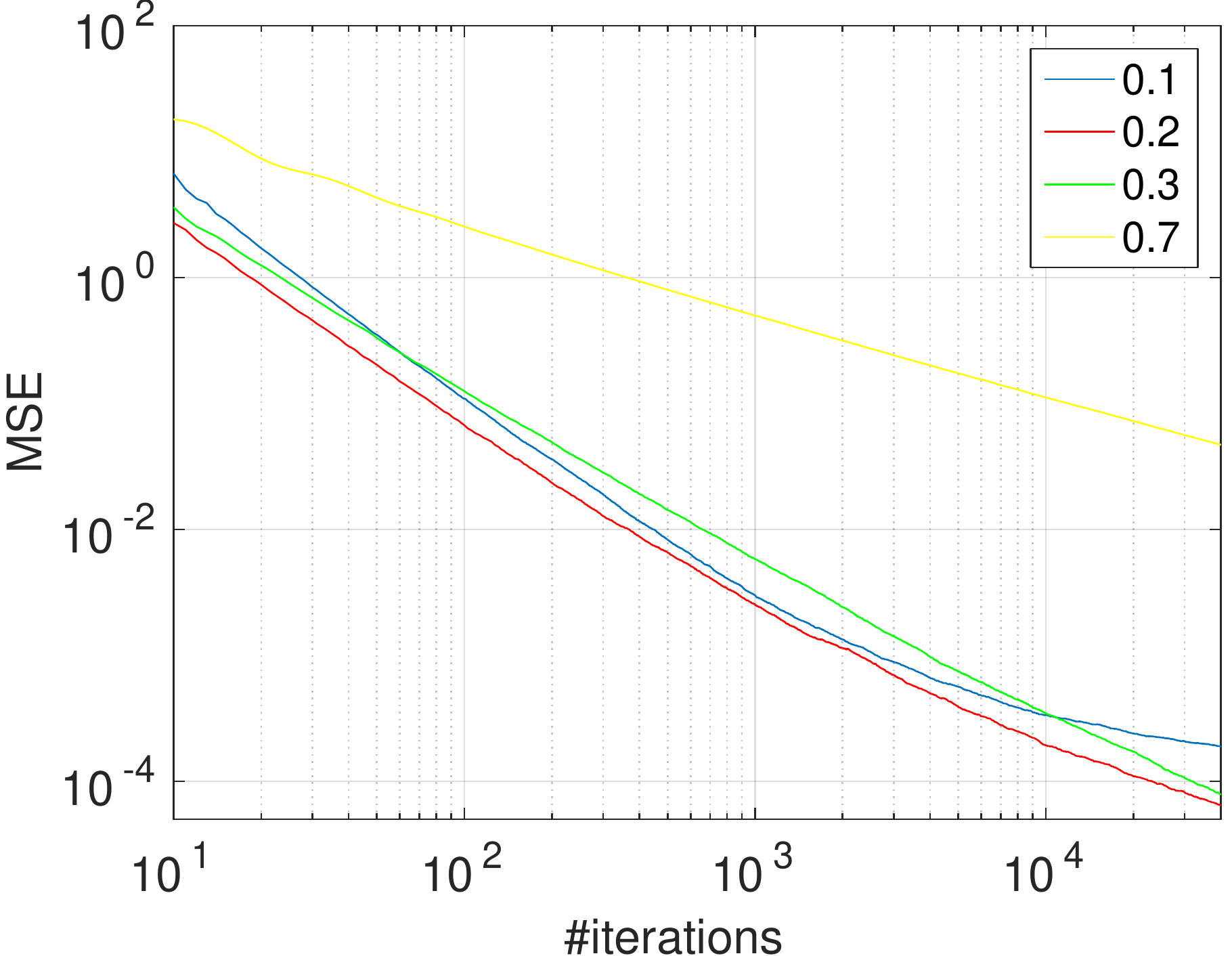}
}
\caption{{\em Bias} and {\em MSE} for decreasing step size SGHMC with different step size rates.}
\label{fig:bias_mse_gau_decrease_rates1}
\end{figure}

\begin{figure}[t!]
\centering
\subfloat[Bias]{
  \includegraphics[width=0.495\linewidth]{BIAS_DECREASE_SGHMC_SGLD.pdf}
}
\subfloat[MSE]{
  \includegraphics[width=0.495\linewidth]{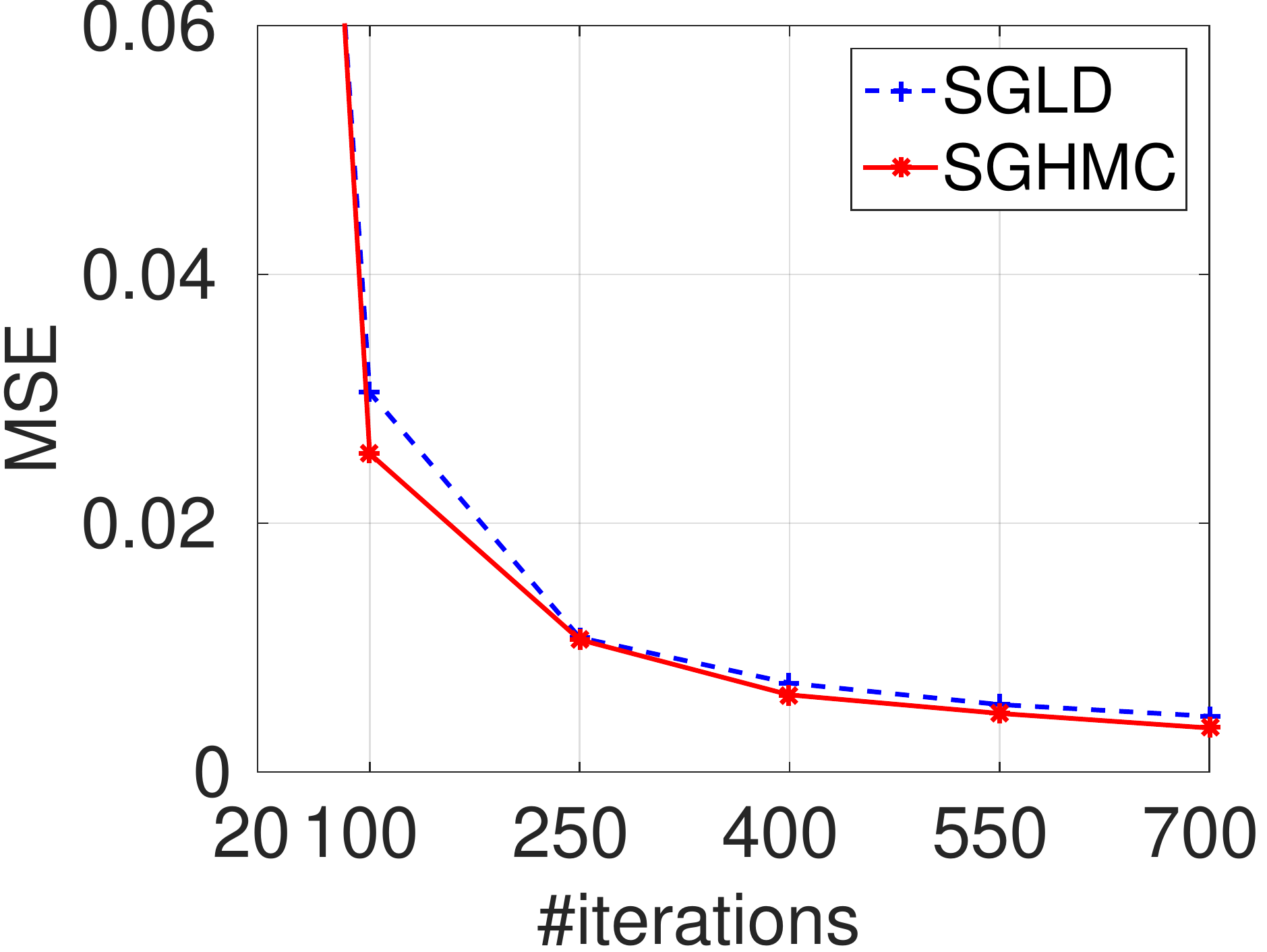}
}
\caption{Comparisons of {\em Bias} and {\em MSE} for SGHMC and SGLD with decreasing step sizes on a simple Gaussian model.}
\label{fig:bias_mse_dec_gau1}
\end{figure} 

\subsection{LDA \& SBN}

We first the list quantitative results of the LDA and SBN models in Table~\ref{tab:LDASBN}. 
It is clear that in both models the SGHMC is much better than the SGLD due to the introduction 
of momentum variables in the dynamics (similar to the SGD with momemtum \cite{ChenFG:ICML14} 
in the optimization literature); and the splitting integrator also works better than the Euler integrator
due to the higher order errors in splitting integrators. For a fair comparison, we did not consider
a better version of the SGLD with Riemannian information geometry of posterior distributions on
probabilistic simplexes \cite{PattersonT:NIPS13}.

\begin{table}[hbpt]
\centering
\caption{Comparisons for different algorithms. $K$ in LDA means \#topics, $J$ in SBN 
	means \#hidden units; suffix `S' means the {\em symmetric splitting integrator}, 
	`E' means the {\em Euler integrator}.}\vspace{-0.3cm}
  \begin{tabular}{c || c | c | c || c || c | c | c}
    \hline
     \multicolumn{4}{c||}{LDA (Test perplexity)} & \multicolumn{4}{|c}{SBN (Test neg-log-likelihood)} \\ \hline \hline
    K  & SGHMC-S & SGHMC-E & SGLD-E & J & SGHMC-S & SGHMC-E & SGLD-E \\ \hline
    200 & {\bf 1168} & 1180 & 2496 & 100 & {\bf 103} & 105 & 126 \\ \hline
    500 & {\bf 1157} & 1187 & 2511 & 200 & {\bf 98} & 100 & 110 \\ \hline
  \end{tabular}
\label{tab:LDASBN}
\end{table}

Next a plot of the test perplexities decreasing with the number of documents processed for the whole dataset 
is given in Figure~\ref{fig:LDA_euler_split} (top), for a comparison of the Euler integrator and the proposed
symmetric splitting integrator. We can see that the symmetric splitting integrator decreases faster than the
Euler integrator. Furthermore, the dictionary learned by the SGHMC with the symmetric splitting integrator
is also given in Figure~\ref{fig:LDA_euler_split} (bottom).

\begin{figure}[t!]
	\begin{center}
	\begin{minipage}{0.99\linewidth}
		\centering
		\includegraphics[width=0.8\linewidth]{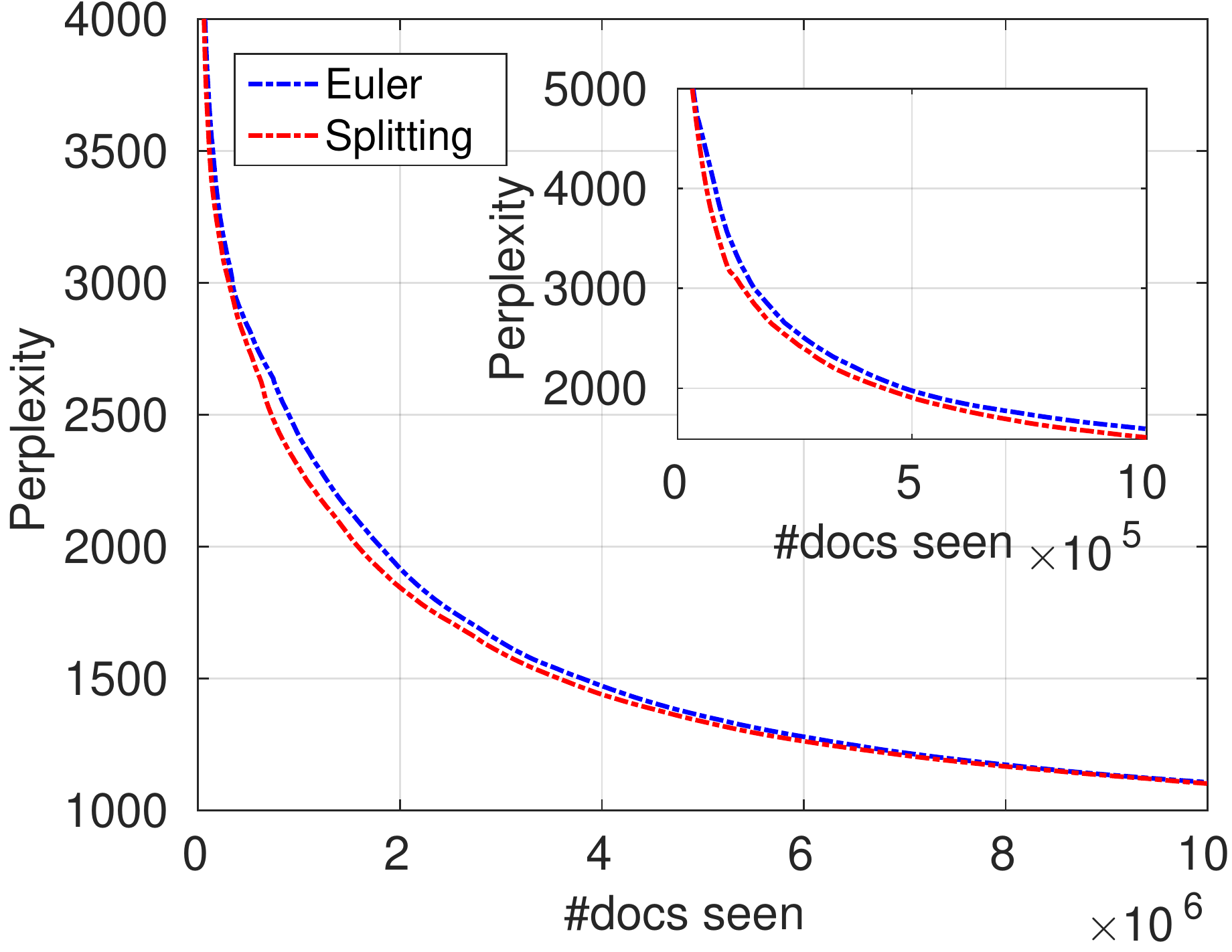}
	\end{minipage}\vspace{2cm}
	\begin{minipage}{0.99\linewidth}
		\centering
		\includegraphics[width=0.7\linewidth]{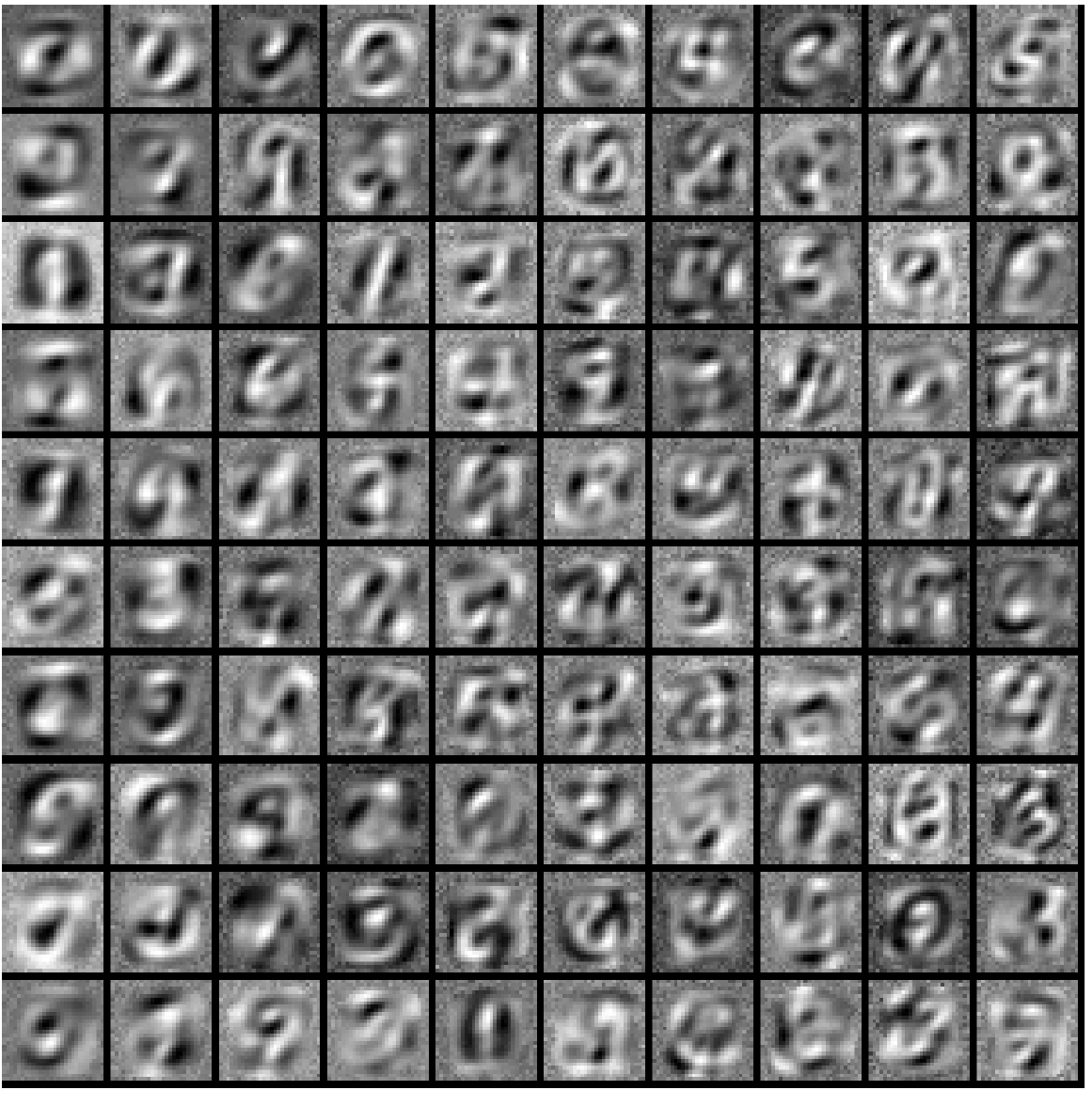}
	\end{minipage}
	\end{center}
\caption{Top: comparisons of Splitting and Euler methods in LDA. Bottom: Dictionary learned by SGHMC in SBN.}
\label{fig:LDA_euler_split}
\end{figure} 

\end{document}